\documentclass[10pt]{article}
\usepackage[utf8]{inputenc}
\usepackage[margin=1in]{geometry}
\usepackage{rotating}
\usepackage{booktabs}
\usepackage{array}
\usepackage{amsmath, amsfonts, amssymb}
\usepackage[table,xcdraw]{xcolor}
\usepackage[colorlinks,linkcolor=blue,citecolor=blue,urlcolor=blue]{hyperref}
\usepackage{bbm}
\usepackage{algorithm}
\usepackage[noend]{algpseudocode}
\usepackage{subcaption}
\usepackage{graphbox}
\usepackage{accents}
\usepackage{adjustbox} 

\usepackage{tcolorbox}
\definecolor{c1}{cmyk}{0,0.6175,0.8848,0.1490} 
\definecolor{c2}{cmyk}{0.1127,0.6690,0,0.4431} 
\definecolor{c3}{cmyk}{0.3081,0,0.7209,0.3255} 
\definecolor{c4}{cmyk}{0.6765,0.2017,0,0.0667} 
\definecolor{c5}{cmyk}{0,0.8765,0.7099,0.3647} 
\definecolor{forestgreen}{HTML}{397727}

\newtcbox{\hlprimarytab}{on line, rounded corners, box align=base, colback=c3!10,colframe=white,size=fbox,arc=3pt, before upper=\strut, top=-2pt, bottom=-4pt, left=-2pt, right=-2pt, boxrule=0pt}
\newtcbox{\hlsecondarytab}{on line, box align=base, colback=red!10,colframe=white,size=fbox,arc=3pt, before upper=\strut, top=-2pt, bottom=-4pt, left=-2pt, right=-2pt, boxrule=0pt}

\newtcbox{\hlorangetab}{on line, box align=base, colback=orange!10,colframe=white,size=fbox,arc=3pt, before upper=\strut, top=-2pt, bottom=-4pt, left=-2pt, right=-2pt, boxrule=0pt}

\newtcbox{\hlgraytab}{on line, rounded corners, box align=base,colframe=white,size=fbox,arc=3pt, before upper=\strut, top=-2pt, bottom=-4pt, left=-2pt, right=-2pt, boxrule=0pt}

\usepackage{hhline}
\usepackage{colortbl}

\usepackage{tabularx,booktabs}
\newcolumntype{Y}{>{\centering\arraybackslash}X}

\usepackage{makecell}

\usepackage[numbers,sort]{natbib}

\usepackage{mathtools}
\mathtoolsset{showonlyrefs}

\usepackage{amsthm}
\usepackage[normalem]{ulem}
\usepackage{soul}

\newtheorem{theorem}{Theorem}

\newtheorem{cor}{Corollary} 

\usepackage{chngcntr}
\usepackage{apptools}
\AtAppendix{\counterwithin{theorem}{section}}
\AtAppendix{\counterwithin{prop}{section}}
\AtAppendix{\counterwithin{cor}{section}}
\AtAppendix{\counterwithin{lemma}{section}}

\usepackage{mathtools}

\usepackage{multicol}
\usepackage{multirow}

\definecolor{Gray}{gray}{0.8}
\definecolor{LightGray}{gray}{0.95}

\def\E{\mathbb{E}}

\renewcommand{\P}{\mathbb{P}}

\newcommand{\Lub}[1]{\hat L^{#1}}

\usepackage{esvect}

\DeclareMathOperator*{\argmin}{arg\,min}

\makeatletter
\def\blfootnote{\xdef\@thefnmark{}\@footnotetext}
\makeatother

\definecolor{red2}{rgb}{0.7, 0, 0.1}

\title{Probably Approximately Correct Labels}

\author{Emmanuel J. Cand\`es\thanks{Authors ordered alphabetically.}$^{*,\dagger, \vartriangle}$ \qquad Andrew Ilyas$^{*,\dagger}$ \qquad Tijana Zrnic$^{*,\dagger, \diamond}$\\ \\ 
$^\dagger$Department of Statistics\\
$^\vartriangle$Department of Mathematics\\
$^\diamond$Stanford Data Science\\ \\
Stanford University}
\date{}

\begin{document}

\maketitle

\begin{abstract}
Obtaining high-quality labeled datasets is often costly, requiring either human annotation or expensive experiments. In theory, powerful pre-trained AI models provide an opportunity to automatically label datasets and save costs. Unfortunately, these models come with no guarantees on their accuracy, making wholesale replacement of manual labeling impractical. In this work, we propose a method for leveraging pre-trained AI models to curate cost-effective and high-quality datasets. In particular, our approach results in \emph{probably approximately correct labels}: with high probability, the overall labeling error is small. Our method is nonasymptotically valid under minimal assumptions on the dataset or the AI model being studied, and thus enables rigorous yet efficient dataset curation using modern AI models. We demonstrate the benefits of the methodology through text annotation with large language models, image labeling with pre-trained vision models, and protein folding analysis with AlphaFold.
\end{abstract}

\section{Introduction}
A key ingredient in any scientific pipeline
is the availability of large amounts of high-quality {\em labeled} data. 
For example, 
social scientists rely on extensively-labeled
datasets to understand human behavior \citep{salganik2017bit}
and design policy interventions. 
Collecting 
high-quality labels for a given set of inputs
is typically an arduous task that requires 
significant human expertise, 
costly large-scale experimentation,
or expensive simulations.
As such, researchers often outsource label collection
to a third party ``data provider''---this might be an 
annotation platform for labeling images, 
a wet lab for running scientific experiments, 
or a survey platform for collecting responses from a 
target population of individuals.

For data providers,
the high cost of collecting high-quality labels
combined with the rising performance of AI models 
suggests an enticing prospect: using AI {\em predictions} 
in place of manually-collected labels.
Indeed, recent works have demonstrated AI models' ability to predict protein structures \citep{jumper2021highly},
to evaluate language model responses \citep{zheng2023judging}, and even to simulate human experimental 
subjects \citep{argyle2023out}.
These advances highlight the potential for AI to streamline data annotation,
and to produce high-quality labels at a fraction of the cost.

The problem with such an approach is that AI models are not
always accurate, and come with no guarantees on 
how well they will label a given dataset. 
This makes it untenable to use
AI-predicted labels as a direct substitute for expert labels, particularly in
settings where label quality is critical. For instance, if 
the downstream goal is to draw conclusions that inform policy decisions,
we should not blindly treat AI predictions of human behavior as 
if they were experimentally collected data.

Motivated by this state of affairs, in this paper we ask:
\begin{center}
    \textit{Can we leverage powerful AI models to label data, while still guaranteeing quality?}
\end{center}
We answer this question in the affirmative, and provide a method---which we call \emph{probably approximately correct} (PAC) labeling---that 
automatically combines cheap, non-expert labels (whether AI predictions,
crowd-sourced labels, or simple heuristics) with expensive, expert labels
to produce a labeled dataset with small error.
PAC labeling yields guarantees similar in flavor to that of its namesake in
probably approximately correct (PAC) learning \citep{valiant1984theory}: given 
user-specified constants $\epsilon, \alpha > 0$, our procedure results in a labeled dataset
with error at most $\epsilon$, with
probability at least~$1-\alpha$. This guarantee is \emph{nonasymptotic} under minimal assumptions on the dataset or the predicted labels being used.

\subsection{Contributions}

We give a brief overview of our contributions, beginning with the problem setup. Given an unlabeled dataset $X_1,\dots,X_n \in \mathcal X$, with unknown expert labels $Y_1,\dots,Y_n$, our goal is to return a labeled dataset $(X_1,\tilde Y_1),\dots,(X_n,\tilde Y_n)$, such that we incur only a small amount of labeling errors:
\begin{equation}
\label{eq:labeling_guarantee}
\frac{1}{n} \sum_{i=1}^n \ell(Y_i,\tilde Y_i)\leq\epsilon, \text{ with probability } 1-\alpha.
\end{equation}
Here, $\alpha$ and $\epsilon$ are user-chosen error parameters and $\ell$ is a relevant error metric. For example, if we want categorical labels to be accurate, we can choose the 0-1 loss: $\ell(Y_i, \tilde Y_i) = \mathbf{1}\{Y_i \neq \tilde Y_i\}$. The guarantee \eqref{eq:labeling_guarantee} then requires that at most an $\epsilon$-fraction of the dataset is mislabeled, with high probability. In regression problems, one might choose the squared loss, $\ell(Y_i,\tilde Y_i) = (Y_i - \tilde Y_i)^2$. We call $\tilde Y_i$ that satisfy the criterion \eqref{eq:labeling_guarantee} \emph{probably approximately correct} (PAC) labels.
To avoid making strong assumptions, we treat the data as \emph{fixed}; probabilities are taken only over the labeling algorithm.

To produce the label $\tilde Y_i$, we are allowed to query an expert for $Y_i$, which is costly, or instead use a cheap AI prediction $\hat Y_i = f(X_i)$, where $f$ is an AI model. The prediction $\hat Y_i$ can depend on any feature information available for point $i$, as well as any source of randomness internal to $f$. We will consider two settings: a basic setting with a single AI model $f$, and a more complex setting that assumes access to $k$ different models $f_1,\dots,f_k$.

Of course, we can trivially achieve \eqref{eq:labeling_guarantee} by collecting expert labels for all $n$ data points. The goal is to achieve the criterion while minimizing the cost of the labeling. We will consider two ways of measuring the cost. The basic one is to simply count the number of collected expert labels; the AI-predicted labels are assumed to essentially come at no cost. The second way of measuring the cost takes into account the costs $c_1,\dots,c_k$ of querying the $k$ models, as well as the cost of an expert label $c_{\mathrm{expert}}$. When $c_{\mathrm{expert}}$ is much larger than $c_1,\dots,c_k$, the second setting reduces to the first.

Our main contribution is a method for producing PAC labels which, as we will show through a series of examples with different data modalities and AI models, allow for significant saves in labeling cost. The key feature that enables a cost reduction is access to a good measure of model uncertainty about the label, which allows focusing the expert budget on instances where the model is most uncertain. Crucially, the nonasymptotic validity of PAC labeling does \emph{not} depend on the quality of the uncertainty measure; however, more useful measures lead to larger saves in cost. We provide refinements of the method that additionally learn to calibrate the uncertainty scores to make the saves in cost even more pronounced.

\subsection{Related work}

\paragraph{Adaptive dataset labeling and curation.}
Our work most closely relates to the literature on efficient dataset labeling from possibly noisy labels. A distinguishing feature of our work is that we construct \emph{provably accurate} labels with nonasymptotic guarantees, under no assumptions on the noisy labels. In contrast, much of existing work makes strong parametric or distributional assumptions---for example, model errors following a truncated power-law distribution \cite{qiu2020mcal}, the data following a well-specified parametric family \cite{ratner2016data}, or a class-conditional noise process \cite{northcutt2021confident}. Many works lack formal accuracy guarantees \cite{zhu2002learning, iscen2019label, bernhardt2022active, li2023coannotating, xie2020self}. Since we do not place distributional assumptions on the data but instead consider it fixed, our work particularly relates to the labeling problem known as transductive learning \cite{vapnik1998statistical, joachims2003transductive}.
A key feature of our work is that we leverage pre-trained AI models, such as off-the-shelf language or vision models, and make no complexity assumptions on the expert labeling mechanism. An emerging line of work studies human-AI collaborative approaches to dataset curation \cite{li2023coannotating, yuan2021synthbio, liu2022wanli}. Our work is motivated by similar problems, with a focus on ensuring statistical validity.
Importantly, many of the above works use uncertainty to decide which labels to collect~\cite{bernhardt2022active, li2023coannotating}. Our work similarly relies on uncertainty; in fact, our procedure can be applied as a wrapper around \emph{any} uncertainty score to provide a statistically valid labeling. For example, the CoAnnotating paradigm defines an uncertainty score and proposes annotating the top $k$ most uncertain points with human annotations and the rest with AI annotations, for some user-chosen $k$. Our procedure can be applied to select $k$ in a data-driven manner, so that the final labeling is $(1-\epsilon)$-accurate with high probability. More distant but related is a vast line of work studying different strategies for reliable aggregation of multiple noisy labels \cite{karger2014budget, cheng2022many, dawid1979maximum, whitehill2009whose, zhang2015active, yan2010modeling, welinder2010multidimensional, sheng2008get, yan2011active}. Our focus is on pre-trained AI models as multiple noisy labelers with varying qualities and strengths.

\paragraph{Distribution-free uncertainty quantification.}
At a technical level, 
our procedure resembles the construction of risk-controlling prediction sets \cite{bates2021distribution} and performing risk-limiting audits \cite{waudby2021rilacs, shekhar2023risk}. Like the former, our procedure bounds a monotone loss function by tuning a one-dimensional threshold, though not for the purpose of predictive inference. Similarly to the latter, our procedure aims to collect sufficient expert labels so as to meet a pre-specified quality guarantee. Like all these methods, PAC labeling satisfies \emph{nonasymptotic}, \emph{distribution-free} statistical guarantees. To achieve this, we build on betting-based confidence intervals \cite{waudby2020estimating, orabona2023tight}. 
Our proposal relates in spirit to prediction-powered inference~\cite{angelopoulos2023prediction, zrnic2023cross, angelopoulos2023ppipp} and related control-variate approaches \cite{zhouaccelerating, egami2023using}, where the goal is to improve the power of statistical inferences given a small amount of expert-labeled data, a large amount of unlabeled data, and a good predictive model. We do not focus on statistical inference per se; rather, we aim to construct an accurately labeled dataset that can be used for any downstream task.

\paragraph{Active learning and inference.}

The idea behind our method is to collect expert
labels where the AI model is most uncertain; in that sense, our method relates
to active learning \cite{settles2009active, lewis1995sequential,
beluch2018power, zhang2015active} and active inference \cite{zrnic2024active,
gligoric2024can}. Notably, there is a line of work in active learning that
considers costs \cite{settles2008active_cost, donmez2008proactive,
wang2016cost}.  Our goal is fundamentally different: it is neither fitting a predictive model nor
statistical inference, but producing high-quality labeled data with a provable nonasymptotic guarantee under minimal assumptions. In general, this is neither necessary nor sufficient for active learning.

\section{PAC labeling: core method}
\label{sec:pac_labeling}

We begin with the basic setting with one AI model that produces cheap labels. Thus, we have $\hat Y_i = f(X_i)$ for all data points. In addition, we assume access to scalar uncertainty scores $U_1,\dots,U_n$ (typically scaled such that $U_i \in [0, 1]$) corresponding to the predictions $\hat Y_1,\dots,\hat Y_n$. We place no assumptions on the quality of $U_i$, however if lower $U_i$ correspond to more accurate predictions $\hat Y_i$, the procedure will achieve big gains. The PAC guarantee \eqref{eq:labeling_guarantee} holds no matter the quality of $U_i$.

The basic idea behind the procedure is to find an uncertainty threshold $\hat u$ and label all data points with uncertainty that exceeds this threshold, $U_i \geq \hat u$. The more accurate the predictions $\hat Y_i$ are, the higher this threshold will be.
To explain how we set $\hat u$, we introduce some notation. Let $\ell^u(Y_i,\hat Y_i) = \ell(Y_i,\hat Y_i) \mathbf{1}\{U_i \leq u\}$ and $L^u = \frac 1 n \sum_{i=1}^n \ell^u(Y_i,\hat Y_i)$. Ideally, if we knew $L^u$ for every $u$, we would choose the \emph{oracle threshold}:
\[u^* = \min\left\{U_i : L^{U_i} > \epsilon\right\}.\]
In other words, if we label all points with $U_i \geq u^*$, meaning $\tilde Y_i = Y_i \mathbf{1}\{U_i \geq u^*\} + \hat Y_i \mathbf{1}\{U_i < u^*\}$, then we satisfy $\frac 1 n \sum_{i=1}^n \ell(Y_i, \tilde Y_i) \leq \epsilon$ with probability one. The issue is that we do not have access to $Y_i$, and thus we cannot compute $L^{U_i}$. To resolve this issue, we estimate an upper bound on $L^{U_i}$ by initially collecting expert labels for a small subset of the data. We will soon explain such a strategy; for now assume that for every $\alpha\in(0,1)$ and every $u$, we can obtain a valid upper confidence bound on $L^u$ at level $1-\alpha$, denoted $\Lub{u}(\alpha)$:
\[\P(L^u \leq \Lub{u}(\alpha)) \geq 1-\alpha.\]
Note that we only require $\Lub{u}(\alpha)$ to be valid one $u$ at a time, not simultaneously.
Our empirical approximation of the oracle threshold is given by:
\begin{equation}
\label{eq:uhat}
\hat u = \min \{U_i :  \Lub{U_i}(\alpha) > \epsilon\}.
\end{equation}
Therefore, we collect expert labels where our uncertainty is $\hat u$ or higher:
$\tilde Y_i = Y_i \mathbf{1}\{U_i \geq \hat u\} + \hat Y_i \mathbf{1}\{U_i < \hat u\}$. Figure \ref{fig:PAC} illustrates the procedure visually.
We argue that such labels $\tilde Y_i$ are PAC labels.

\begin{figure}[h]
\centering
\includegraphics[width = 0.6\textwidth]{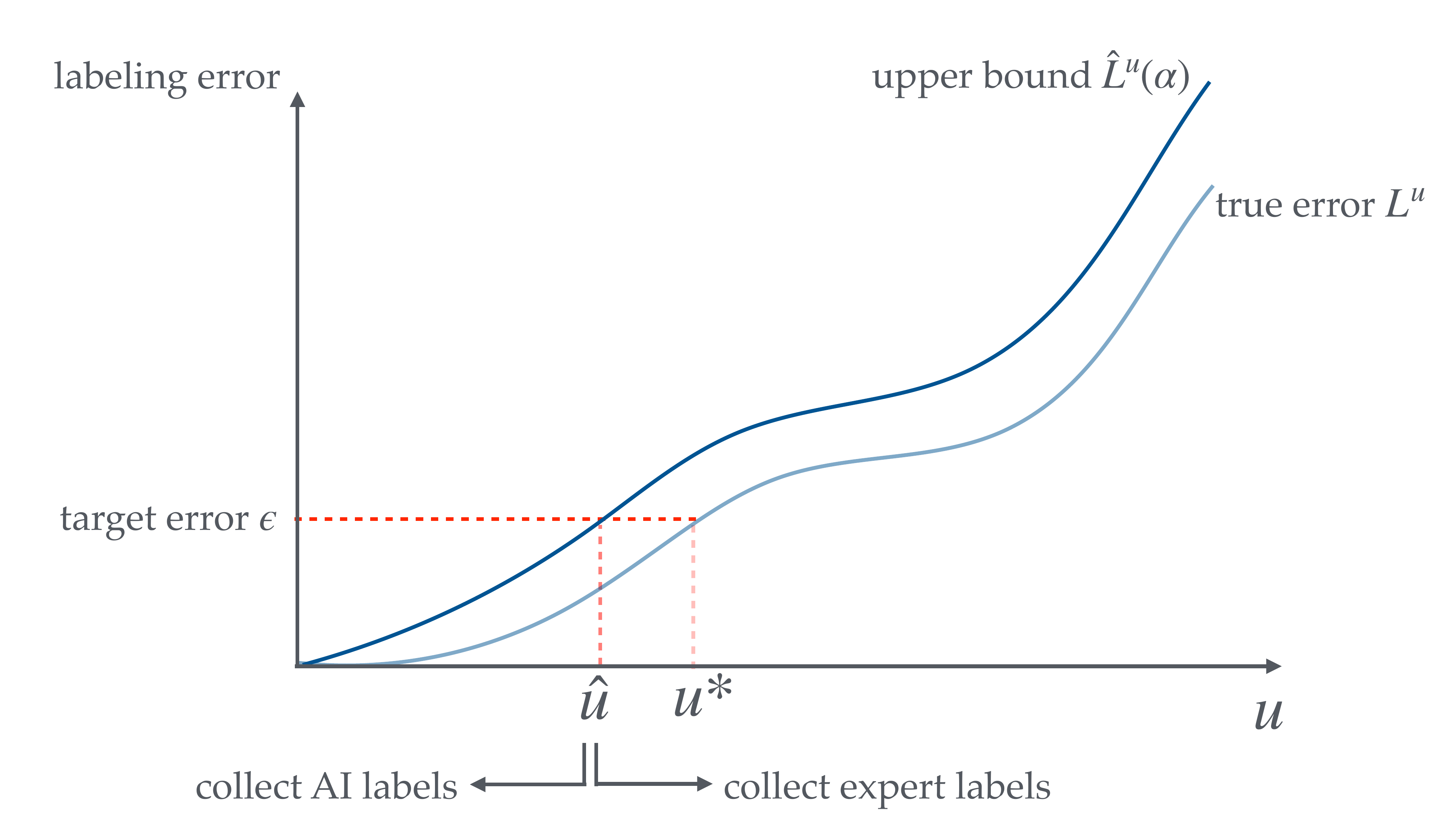}
\caption{\textbf{Illustration of PAC labeling.} The procedure estimates an uncertainty threshold $\hat u$ and collects expert labels for all points where $U_i \geq \hat u$.}
\label{fig:PAC}
\end{figure}

\begin{theorem}
\label{thm:main_theorem}
The labels $\tilde Y_i = Y_i \mathbf{1}\{U_i \geq \hat u\} + \hat Y_i \mathbf{1}\{U_i < \hat u\}$, with $\hat u$ given by~\eqref{eq:uhat}, are PAC labels~\eqref{eq:labeling_guarantee}.
\end{theorem}

\begin{proof}
By the definition of $u^*$, we know $\frac 1 n \sum_{i=1}^n \ell(Y_i, \tilde Y_i)\leq \epsilon$ if $\tilde Y_i = Y_i \mathbf{1}\{U_i \geq  u^*\} + \hat Y_i \mathbf{1}\{U_i < u^*\}$. Furthermore, by monotonicity, for any labeling threshold $u'\leq u^*$ the error criterion is satisfied. Therefore, on the event that $\hat u\leq u^*$, we know that  $\frac 1 n \sum_{i=1}^n \ell(Y_i, \tilde Y_i)\leq \epsilon$.

We argue that $\P(\hat u \leq u^*) \geq 1-\alpha$ as long as $\Lub{U_i}(\alpha)$ are valid upper confidence bounds for all $U_i$. Suppose not: suppose $\hat u > u^*$. By definition, this must mean that $\Lub{u^*}(\alpha)\leq \epsilon$. But at the same time, we know $L^{u^*}>\epsilon$; therefore, it must be that $\Lub{u^*}(\alpha) < L^{u^*}$. This event happens with probability at most $\alpha$ because $\Lub{u^*}(\alpha)$ is a valid upper confidence bound,  and thus we have shown $\P(\hat u \leq u^*) \geq 1-\alpha$.
\end{proof}

Interestingly, notice that the proof only requires $\hat L^{U_i}(\alpha)$ to be valid \emph{individually}, even though we form $n$ confidence bounds. This is a consequence of the monotonicity of $L^u$ in $u$, similar in spirit to how monotonicity enables the Dvoretzky–Kiefer–Wolfowitz inequality~\cite{dvoretzky1956asymptotic} and risk-controlling prediction sets~\cite{bates2021distribution} to be free of multiplicity corrections.

\begin{algorithm}[t]
\caption{Probably Approximately Correct Labeling}
\label{alg:pac_labeling}
\begin{algorithmic}[1]
\Require unlabeled data $X_1,\dots,X_n$, predicted labels $\hat Y_1,\dots,\hat Y_n$, uncertainties $U_1,\dots,U_n$, labeling error $\epsilon$, error probability $\alpha\in(0,1)$, sample size for estimation $m$, sampling weights $\pi_1,\dots,\pi_n$
\State Sample $i_j\sim \mathrm{Unif}\left([n]\right)$ and $\xi_{i_j} \sim \mathrm{Bern}(\pi_{i_j})$ independently for $j\in[m]$
\State Collect $Y_{i_j}$ if $\xi_{i_j}=1$ for $j\in[m]$
\State Compute confidence bound $\hat L^u(\alpha) = \texttt{meanUB}\left(\{\ell^u(Y_{i_j}, \hat Y_{i_j}) \frac{\xi_{i_j}}{\pi_{i_j}}\}_{j\in[m]}; \alpha\right)$ for all $u \in \{U_i\}_{i=1}^n$
\State Let $\hat u = \min\{U_i: \hat L^{U_i}(\alpha) > \epsilon\}$
\State Collect true labels $Y_i$ for points where $U_i \geq \hat u$
\State Let $\tilde Y_{i} \leftarrow Y_{i} \mathbf{1}\{U_i \geq \hat u\} + \hat Y_i \mathbf{1}\{U_i < \hat u\}$ for all $i\in[n]$
\State For all $\{i_j\}_{j\in[m]}$ s.t. $\xi_{i_j} = 1$, (possibly) update $\tilde Y_{i_j}\leftarrow Y_{i_j}$
\Ensure labeled dataset $(X_1, \tilde Y_1),\dots,(X_n, \tilde Y_n)$
\end{algorithmic}
\end{algorithm}

It remains to provide a method to compute $\Lub{U_i}(\alpha)$. Given a hyperparameter $m$, we collect $m$ draws $\{i_1,\dots,i_m\}$ independently as $i_j \sim \mathrm{Unif}([n])$. Then, for all $j\in[m]$, we sample $\xi_{i_j} \sim \mathrm{Bern}(\pi_{i_j})$, where $(\pi_1,\dots,\pi_n)$ are arbitrary sampling weights, and collect $Y_{i_j}$ if $\xi_{i_j}=1$. This results in a dataset of $m$ i.i.d. variables $\left\{\ell(Y_{i_j}, \hat Y_{i_j})\frac{\xi_{i_j}}{\pi_{i_j}}\right\}_{j=1}^m$;
therefore, we can estimate $\Lub{u}(\alpha)$ as:
$$\Lub{u}(\alpha) = \texttt{meanUB}\left(\left\{\ell(Y_{i_j}, \hat Y_{i_j})\frac{\xi_{i_j}}{\pi_{i_j}} \mathbf{1}\{U_{i_j} \leq u\}\right\}_{j=1}^m;\alpha\right).$$
Here, $\texttt{meanUB}(\cdot;\alpha)$ is any method for computing a valid upper bound at level $1-\alpha$ on the mean from an i.i.d. sample. 
Indeed, the samples $\ell(Y_{i_j}, \hat Y_{i_j})\frac{\xi_{i_j}}{\pi_{i_j}} \mathbf{1}\{U_{i_j} \leq u\}$ are i.i.d. with mean $L^u$, since $\E[\xi_{i_j}/\pi_{i_j} | i_j] = 1$. The motivation for allowing adaptive sampling weights $\pi_i$ is to allow forming a tighter confidence bound through a careful choice of the weights, although even uniform weights $\pi_1=\dots=\pi_n = p \in (0,1)$ are a reasonable choice in practice.

There are many possible choices for $\texttt{meanUB}(\cdot;\alpha)$. For instance, this can be a nonasymptotic procedure such as the betting-based confidence intervals \cite{waudby2020estimating, shekhar2023risk}. If one is satisfied with asymptotic guarantees, then one can simply compute a confidence bound based on the central limit theorem:
\begin{equation}
\label{eq:CLT_ub}
\texttt{meanUB}(\{Z_j\}_{j=1}^m;\alpha) = \hat \mu_Z + z_{1-\alpha}  \frac{\hat \sigma_Z}{\sqrt{m}},
\end{equation}
where $\hat \mu_Z$ and $\hat \sigma_Z$ are the empirical mean and standard deviation of $\{Z_j\}_{j=1}^m$, respectively, and $z_{1-\alpha}$ is the $(1-\alpha)$-quantile of the standard normal distribution. In our experiments, we will primarily focus on procedures with nonasymptotic validity.

We summarize the overall procedure in Algorithm \ref{alg:pac_labeling} and its guarantee in Corollary \ref{cor:pac_labeling}.

\begin{cor}
\label{cor:pac_labeling}
For any valid mean upper bound subroutine \textup{\texttt{meanUB}}, Algorithm \ref{alg:pac_labeling} outputs PAC labels.
\end{cor}

\subsection{Uncertainty calibration}
\label{sec:uncertainty_calibration}
The performance of PAC labeling crucially depends on the quality of the uncertainty scores. However, some data points $X_i$ might have more accurate uncertainties than others. For example, suppose we can partition the $X_i$ into two groups: on one, the model is consistently overconfident, and on the other, the model is consistently underconfident. Then, PAC labeling will overcollect expert labels for the data points in the second group. In the extreme case, imagine the model is always incorrect on data points from the first group but produces low uncertainties, and is always correct on data points from the second group but produces high uncertainties. Then, all expert labels for the second group will be collected (except in trivial cases when $\epsilon$ is too large or the second group is too small). This is clearly wasteful, especially if the second group is of significant size.

We propose uncertainty calibration as a way of mitigating this issue. One natural way of calibrating uncertainties arises when there exists a collection $\mathcal C$ of possibly overlapping clusters in the data, where each $C\in \mathcal C$ is a collection of data point indices. These clusters could be implied by externally given features (such as demographic features), or they could be discovered in a data-driven way. For the zero--one loss, we use the multicalibration algorithm from \citet{hebert2018multicalibration}, stated in Algorithm \ref{alg:calibrate} for completeness, to learn the uncertainty adjustment for each cluster. In practice, we learn the adjustment by collecting expert labels for a small subset of size $m\ll n$ of the overall dataset and applying the correction to the remainder of the dataset.

\begin{algorithm}[h]
\caption{Uncertainty Multicalibration \cite{hebert2018multicalibration}}
\label{alg:calibrate}
\begin{algorithmic}[1]
\Require uncertainties $U_1,\dots,U_m \in [0,1]$, expert labels $Y_1,\dots,Y_m$, predicted labels $\hat Y_1,\dots,\hat Y_m$, clusters $\mathcal{C}$, number of bins $B$, tolerance $\tau > 0$
\State Define bins $b_j = \left[\frac{j-1}{B}, \frac{j}{B}\right)$ for $j = 1,\dots,B$
\Repeat
    \State $\texttt{updated} \gets \textbf{False}$
    \For{each cluster $C \in \mathcal{C}$ and each bin $j = 1,\dots,B$}
        \State Let $\mathcal{I}^{C,j} = \{i \in C : U_i \in b_j\}$
        \If{$|\mathcal{I}^{C,j}| > 0$}
            \State Compute correction: $\Delta_{C,j} \gets \frac{1}{|\mathcal{I}^{C,j}|} \sum_{i \in \mathcal{I}^{C,j}} \left( \mathbf{1}\{Y_i \neq \hat Y_i\} - U_i \right)$
            \If{$|\Delta_{C,j}| > \tau$}
                \State Update: $U_i \gets U_i + \Delta_{C,j}$ for all $i \in \mathcal{I}^{C,j}$
                \State $\texttt{updated} \gets \textbf{True}$
            \EndIf
        \EndIf
    \EndFor
\Until{\texttt{updated} is \textbf{False}}
\Ensure calibrated uncertainties $U_1,\dots,U_m$
\end{algorithmic}
\end{algorithm}

\section{Multi-model labeling via the PAC router}
\label{sec:router}
In many cases, we have access to several different sources of 
non-expert predictions. 
For example, we might have labels from several different AI models, 
or from (non-expert) human annotators of varying skill levels. 
In such settings, we might hope to leverage the strengths of 
these different predictors to reduce our overall labeling cost.

Concretely, consider a setting with $k$ cheap labeling sources;
for each data point $i$, each source $\smash{j \in [k]}$ provides a predicted label
$\smash{\hat Y_i^j}$ and an uncertainty $\smash{U_i^j}$. 
Our goal is to route each data point to the most reliable source,
minimizing the number of expert labels that we need 
to collect to retain the PAC guarantee 
\eqref{eq:labeling_guarantee}. 
(We later move to a cost-sensitive setting.)
Our high-level approach is in two steps:
\begin{enumerate}
    \item First, we will learn a {\em routing model} $\smash{w_\theta: \mathcal X \to \Delta^{k-1}}$ that
    maps each data point to a distribution over the $k$ labeling sources. We use the routing model
    to find the best source $\smash{j^*_i}$ for each data point $\smash{i}$, 
    to which we assign label $\smash{\hat Y_i = \hat Y_i^{j^*_i}}$ and uncertainty $\smash{U_i = U_i^{j^*_i}}$.
    \item We then apply the PAC labeling procedure from Section \ref{sec:pac_labeling} 
    to the selected data points, using the routed labels and uncertainties.
\end{enumerate}
The main question is how to learn the routing model $w_\theta$.
Throughout, we will assume access to a small, fully labeled {\em routing dataset} of size $m$,
for which we observe $\smash{(X_i, Y_i, \{\hat{Y}_i^j, U_i^j\}_{j=1}^k)_{i=1}^m}$,
which we can use to learn the routing model.

A natural first idea (but ultimately a suboptimal one) is to maximize the
expected accuracy of the routed labels---i.e., to solve $
    \argmin_\theta \sum_{i=1}^m \sum_{j=1}^k w_{\theta,j}(X_i)   \ell(Y_i, \hat{Y}_i^j),$
where $w_{\theta,j}(X_i)$ denotes the $j$-th coordinate of $w_\theta(X_i)$.
This router is suboptimal because it fails to take into account the models' uncertainties 
as well as our error tolerance $\epsilon$. 
To see why such a router is suboptimal, consider the case where one of the labeling sources
has 100\% accuracy, but also has uniformly higher uncertainty than the other sources.
For the purposes of PAC labeling, this source is not helpful;
indeed, it will result in more expert labels being collected than if we had used the
other sources. The router, however, will be incentivized to route all points to
this source to maximize expected accuracy.

Can we route points in a way that takes into account the ultimate 
cost of the labeling procedure? 
To start, observe that the actual expected cost incurred by 
using a particular routing model $w_\theta$ is
\begin{equation}
\label{eq:uncosted_loss}
\sum_{i=1}^m \sum_{j=1}^k w_{\theta,j}(X_i) \mathbf{1}\{U_i^j \geq \hat u\},
\end{equation}
where $\hat u$ is the threshold set by the PAC labeling procedure.
Ideally, we could minimize this quantity directly, e.g., using gradient descent.
There are two barriers to doing so: first, \eqref{eq:uncosted_loss} is
non-differentiable due to the $\mathbf{1}\{\cdot\}$ term,
and second, the threshold $\hat{u}$ implicitly depends on the routing model $w_\theta$ itself.

To circumvent these issues, we first replace the indicator $\mathbf{1}\{U_i^j > \hat u\}$ with a sigmoid
$\sigma(U_i^j - \hat u)$. We then consider the following differentiable relaxation of the PAC labeling 
scheme that allows us to take gradients of our final objective with respect to the 
parameters of the routing model.

Concretely, we consider a labeling scheme based on a threshold $\tilde{u}$ 
computed in the following way.
We can approximate the PAC labeling guarantee with a weaker guarantee of 
expected average error control, then our procedure for finding $\tilde{u}$
can be written as:
\begin{equation}
\label{eq:router_equation}
    \tilde{u} \approx \min\left\{u : \mathbb{E}_{X_i, Y_i, j \sim w_\theta(X_i)}[\ell(Y_i, \hat{Y}_i^j) \cdot \mathbf{1}\{U_i^j \leq u\}] > \epsilon\right\},
\end{equation}
where the expectation over $X_i,Y_i$ denotes the empirical average over the (fixed) data points $(X_i,Y_i)$. If we again replace the indicator $\mathbf{1}\{U_i^j \leq u\}$ with a sigmoid,
then $\tilde{u}$ is the solution to the equation:
\begin{equation}
    \label{eq:smooth_threshold}
    \mathbb{E}_{X_i, Y_i}\left[ \sum_{j=1}^k w_{\theta,j}(X_i) \cdot \ell(Y_i, \hat{Y}_i^j) \cdot \sigma(\tilde{u} - U_i^j)\right] = \epsilon.
\end{equation}
By strict monotonicity of the sigmoid and positivity of the remaining terms, 
this solution is unique. Therefore, we can write it as $\tilde{u}(\theta)$, 
and use the implicit function theorem to 
compute the gradient of $\tilde{u}$ with respect to the parameters
of the routing model by differentiating both sides of the above equation:
\begin{align*}
    0 = \nabla_\theta \epsilon 
    &= \mathbb{E}_{X_i, Y_i}\left[\nabla_\theta \sum_{j=1}^k w_{\theta,j}(X_i) \cdot \ell(Y_i, \hat{Y}_i^j) \cdot \sigma(\tilde{u}(\theta) - U_i^j)\right] \\
    &= \mathbb{E}_{X_i, Y_i}\left[\sum_{j=1}^k \nabla_\theta w_{\theta,j}(X_i) \cdot \ell(Y_i, \hat{Y}_i^j) \cdot \sigma(\tilde{u}(\theta) - U_i^j) \right. \\ 
    &\qquad \qquad \left.\phantom{\sum_{j=1}^k} + w_{\theta,j}(X_i) \cdot \ell(Y_i, \hat{Y}_i^j) \cdot \sigma(\tilde{u}(\theta) - U_i^j) \cdot (1-\sigma(\tilde{u}(\theta) - U_i^j)) \cdot \nabla_\theta \tilde{u}(\theta)\right].
\end{align*}

Rearranging, we get:
\begin{equation}
    \nabla_\theta \tilde{u}(\theta) =  \frac{-\mathbb{E}_{X_i, Y_i}\left[\sum_{j=1}^k \nabla_\theta w_{\theta,j}(X_i) \cdot \ell(Y_i, \hat{Y}_i^j) \cdot \sigma(u(\theta) - U_i^j)\right]}{\mathbb{E}_{X_i, Y_i}\left[\sum_{j=1}^k w_{\theta,j}(X_i) \cdot \ell(Y_i, \hat{Y}_i^j) \cdot \sigma(u(\theta) - U_i^j) \cdot (1-\sigma(u(\theta) - U_i^j))\right]}.
\end{equation}
We can estimate the above gradient using a single expectation, by defining the 
probability distribution 
over datapoint-model pairs $(i, j)$:
\begin{equation}
    \eta_\theta(i,  j) \propto w_\theta(X_i)_j \cdot \ell(Y_i, \hat{Y}_i^j) \cdot \sigma(\tilde{u}(\theta) - U_i^j) \cdot (1-\sigma(\tilde{u}(\theta) - U_i^j)),
\end{equation}
such that
\begin{equation}
\label{eq:util}
    \nabla_\theta \tilde{u}(\theta) = - \mathbb{E}_{(i, j) \sim \eta_\theta(\cdot)} \left[ \nabla_\theta \log w_\theta(X_i)_j \cdot \frac{1}{1 - \sigma(\tilde{u}(\theta) - U_i^j)} \right].
\end{equation}

This calculation suggests a natural algorithm for training the weighter: we compute the 
``smooth threshold'' $\tilde{u}(\theta)$ by solving \eqref{eq:smooth_threshold} (e.g., via binary search);
we take a gradient step on the objective 
\begin{equation*}
    \sum_{i=1}^n \sum_{j=1}^k w_{\theta,j}(X_i) \cdot 
    \sigma(\tilde{u}(\theta) - U_i^j),
\end{equation*}
using the gradient \eqref{eq:util} to backpropagate through the threshold computation; and finally 
we repeat the above two steps until convergence.


\subsection{Recalibrating uncertainties}
\label{sec:router_learned_uncertainties}

Even with a principled way to route data points to different
models, in practice our performance will often be bottlenecked 
by the quality of the uncertainties $U_i^j$. 
In particular, if all of the models are uncalibrated 
on a given data point, then routing the point to the best-performing source 
will not yield any benefit in terms of the number of expert labels collected.
Furthermore, the uncertainty values do not 
reflect the fact that we have routed the data point to the source we 
expect to be most reliable. 
Motivated by these observations, we propose a procedure for 
simultaneously learning a routing model and a better uncertainty model.
The main idea is exactly the same as before: we will define an 
uncertainty model $u_\gamma : \mathcal X \to [0, 1]$ that maps a data point
to a new uncertainty value. To train the uncertainty model, we will 
use the same smoothed threshold procedure as before, noting 
now that the threshold $\tilde{u} = \tilde{u}(\theta, \gamma)$ depends on both the parameters 
of the routing model and the parameters of the uncertainty model.
Accordingly, we perform gradient descent to solve the optimization problem
\begin{equation}
    \min_{\theta, \gamma} \sum_{i=1}^m \sum_{j=1}^k w_{\theta,j}(X_i) 
    \cdot \sigma(\tilde{u}(\theta, \gamma) - u_\gamma(X_i)),
\end{equation}
using implicit gradients $\nabla_\theta \tilde u(\theta,\gamma)$ and $\nabla_\gamma \tilde u(\theta,\gamma)$:
\begin{align*}
    \nabla_\theta \tilde{u}(\theta, \gamma) = - \mathbb{E}_{(i, j) \sim \eta_\theta} \left[\frac{\nabla_\theta \log w_{\theta,j}(X_i)}{1 - \sigma(\tilde{u}(\theta, \gamma) - U^j_i)} \right]
    \quad \text{and} \quad
    \nabla_\gamma \tilde{u}(\theta, \gamma) = \mathbb{E}_{(i, j) \sim \eta_\theta} \left[\nabla_\gamma u_{\gamma,j}(X_i) \right].
\end{align*}
The implicit gradients are derived using similar logic as before.

\subsection{Cost-sensitive PAC router}
So far, we have treated the $k$ cheap labeling sources as if they are free
(or vanishingly cheap, compared to the cost of the expert labeler).
In practice, however, we may want to take the cost of the labeling sources into account.
For example, these different sources may represent running experiments with different 
numbers of crowd workers, or with public APIs that have different costs.
Suppose each labeling source $j$ has a per-label cost $c_j$, 
and that the cost of the expert labeler is $c_{\mathrm{expert}}$.
To incorporate costs, we use the same idea as the previous two sections,
aiming to directly optimize the expected cost incurred by the labeling procedure.
Our expected cost becomes
\begin{equation}
    \sum_{i=1}^m \mathbb{E}_{j \sim w_\theta(X_i)} \left[ c_j \cdot \mathbf{1}\{U_i^j < \hat u \} + c_{\mathrm{expert}} \cdot \mathbf{1}\{U_i^j \geq \hat u\} \right],
\end{equation}
where $\hat u$ is the threshold computed using the main PAC labeling procedure.
Just as in the previous sections, we will approximate this threshold 
with a smoothed threshold $\tilde{u}$ and use the implicit function theorem to 
derive the gradient of $\tilde{u}$ with respect to the parameters of the routing model
and the uncertainty model. Finally, we replace the indicators in the above objective
with sigmoids to get a fully differentiable objective, and perform gradient descent.

\section{Experiments}
\label{sec:exps}

We evaluate PAC labeling on a series of real datasets, spanning natural language processing, computer vision, and proteomics. 
We repeat each experiment $1000$ times and report the mean and standard deviation of the save in budget, i.e., the percentage of data points that are \emph{not} expert labeled. We also report the $(1-\alpha)$-quantile of the empirical error $\frac 1 n \sum_{i=1}^n \ell(Y_i, \tilde Y_i)$ (which is supposed to be upper bounded by $\epsilon$). We plot the budget save against the realized error for $50$ of the $1000$ trials.
We fix $\alpha=0.05$ throughout and vary $\epsilon$. All of the analyzed datasets come with expert labels collected by the authors of the original study, which we use to evaluate the error of PAC labeling. All code for reproducing these experiments is available at \url{https://github.com/tijana-zrnic/pac-labels/}.

\subsection{PAC labeling with a single model}
\label{sec:exps_single_model}

We begin with the single-model case. In addition to PAC labeling, we consider two baselines. The first is the ``naive'' baseline, which collects expert labels for all points where the model's uncertainty is above a fixed threshold, such as $10\%$ or $5\%$. The second baseline is the method that only uses the AI labels, without using any expert labels.

\paragraph{Discrete labels.} First we study the problem of collecting discrete labels; thus, we use the zero--one loss, $\ell(Y_i, \tilde Y_i) = \mathbf{1}\{Y_i \neq \tilde Y_i\}$.
We consider several text annotation tasks from computational social science: collecting binary labels of whether a text contains misinformation ($Y_i\in\{\texttt{misinfo}, \texttt{real}\}$) \cite{gabriel2022misinfo}, labels of media headline stance on global warming, i.e.~whether the headline agrees that global warming is a serious concern ($Y_i\in\{\texttt{agree}, \texttt{neutral},  \texttt{disagree}\}$) \cite{luo2020detecting}, and labels of political bias of media articles ($Y_i \in\{\texttt{left},  \texttt{center}, \texttt{right}\}$) \cite{baly2020we}. We use predicted labels $\hat Y_i$ from GPT-4o, collected by \citet{gligoric2024can}. For the uncertainties $U_i$, we use GPT's verbalized confidence scores \cite{tian2023just}; that is, we prompt the model to state its confidence in the answer. Additionally, we consider image labeling on ImageNet and ImageNet v2. We use the ResNet-152 from \cite{he2016deep} to obtain $\hat Y_i$, and set $U_i = 1 - p_{\max}(X_i)$, where by $p_{\max}(X_i)$ we denote the maximum softmax output given image $X_i$. We use the betting algorithm of \citet{waudby2020estimating} (Theorem 3) as the mean upper bound subroutine in the algorithm. In Appendix \ref{app:asymptotic} we include analogous results with the simpler, asymptotic mean upper bound \eqref{eq:CLT_ub}.

We summarize the results in Table \ref{tab:classification_text}, Table \ref{tab:classification_vision},  and Figure \ref{fig:classification}. Using a fixed uncertainty threshold such as $5\%$ or $10\%$ results in highly variable results across datasets; sometimes the naive baseline is valid, sometimes it is not, and when it is valid often it is conservative. The approach of using AI labels alone achieves error that is far above the nominal. PAC labeling achieves error that fluctuates tightly around $\epsilon$, and the budget saves range between $14\%$ and $60\%$ depending on the difficulty of the labeling.

\begin{table*}[t!]
\small
\centering
\begin{adjustbox}{center}
\begin{tabular}{l|l|rrrr}
\toprule
\multirow{2}{*}{\textbf{Dataset}}                & \multirow{2}{*}{\textbf{Metric}}                   & \multicolumn{3}{c}{\textbf{Method}}   \\
&  & PAC labeling & Naive ($U_i\geq 0.1$)  & Naive ($U_i\geq 0.05$) & AI only \\
\midrule

\multirow{2}{*}{\textbf{\shortstack[l]{Media bias }}} 
& {Budget save (\%)} & (13.79 $\pm$ 3.38)\% & 17.76\% & 8.35\%  &  --- \\
& {Error} & \hlprimarytab{4.10\%}  & \hlprimarytab{2.95\%} & \hlprimarytab{1.10\%} & \hlsecondarytab{37.72\%}    \\
\midrule             
\multirow{2}{*}{\textbf{\shortstack[l]{Stance on \\ global warming}}} 
& {Budget save (\%)} & (28.09 $\pm$ 3.28)\% & 62.51\% & 25.10\%  &  --- \\
& {Error} & \hlprimarytab{4.57\%} & \hlsecondarytab{10.13\%}  & \hlprimarytab{0.83\%} & \hlsecondarytab{24.79\%}    \\

\midrule
\multirow{2}{*}{\textbf{{\shortstack[l]{Misinformation}}}} 
& {Budget save (\%)} & (18.12 $\pm$ 4.93)\%  & \negthickspace 50.44\% & \negthickspace 2.65\%  & ---  \\
& {Error} & \hlprimarytab{3.80\%} & \hlsecondarytab{7.07\%} & \hlprimarytab{0.10\%} & \hlsecondarytab{18.62\%} \\

\bottomrule
\end{tabular}
\end{adjustbox}
\caption{\textbf{PAC labeling text datasets with GPT-4o.} We set $\epsilon=0.05$. PAC labeling \hlprimarytab{meets the error criterion}, the AI only baseline has a~\hlsecondarytab{large error}, and the fixed threshold baseline is sometimes valid and sometimes not. Even when it is valid, it can be conservative.}
\label{tab:classification_text}
\end{table*}

\begin{table*}[t!]
\small
\centering
\begin{adjustbox}{center}
\begin{tabular}{l|l|rrrr}
\toprule
\multirow{2}{*}{\textbf{Dataset}}                & \multirow{2}{*}{\textbf{Metric}}                   & \multicolumn{3}{c}{\textbf{Method}}   \\
&  & PAC labeling & Naive ($U_i\geq 0.1$)  & Naive ($U_i\geq 0.05$) & AI only \\
\midrule

\multirow{2}{*}{\textbf{{\shortstack[l]{ImageNet}}}} 

& {Budget save (\%)} & (59.64 $\pm$ 1.49)\% &  60.28\%  &  52.79\%  & ---  \\
& {Error} & \hlprimarytab{4.73\%} & \hlprimarytab{3.15\%} & \hlprimarytab{2.00\%} & \hlsecondarytab{21.69\%} \\
\midrule
                               
\multirow{2}{*}{\textbf{\shortstack[l]{ImageNet v2}}} 
& {Budget save (\%)} & 
(39.07 $\pm$ 2.67)\% & 46.05\% & 39.07\%  &  --- \\
& {Error} & \hlprimarytab{4.74\%} & \hlprimarytab{4.31\%}  & \hlprimarytab{2.62\%} & \hlsecondarytab{35.33\%}    \\
\bottomrule
\end{tabular}
\end{adjustbox}
\caption{\textbf{PAC labeling image datasets with ResNet-152.} We set $\epsilon=0.05$. PAC labeling and the fixed threshold baseline \hlprimarytab{meet the error criterion} and the AI only baseline has a~\hlsecondarytab{large error}. Even when it is valid, the fixed threshold baseline can be conservative.}
\label{tab:classification_vision}
\end{table*}

\begin{figure}[h!]
\centering
\includegraphics[width = 0.8\textwidth]{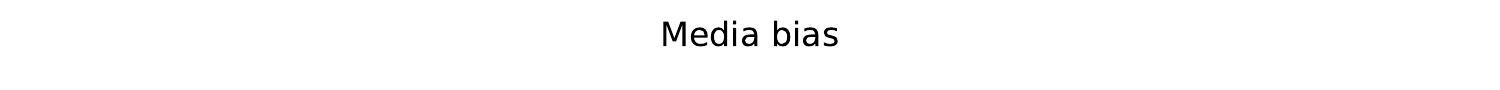}
\includegraphics[width = 0.29\textwidth]{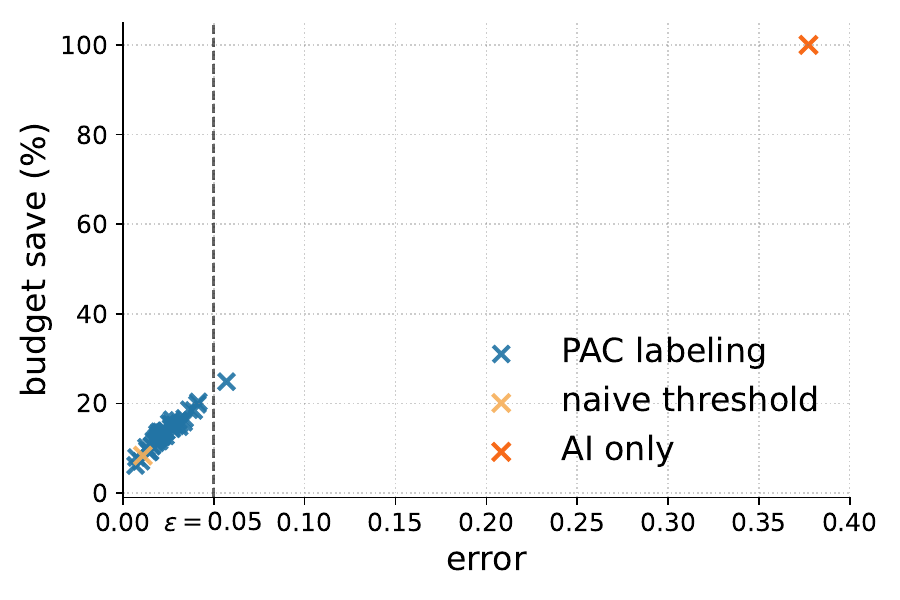}
\includegraphics[width = 0.29\textwidth]{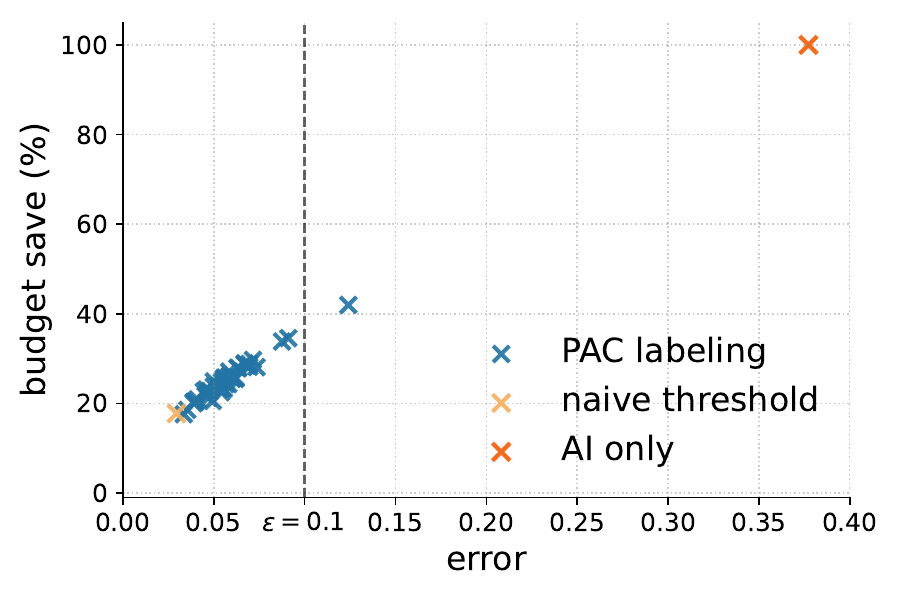}
\includegraphics[width = 0.29\textwidth]{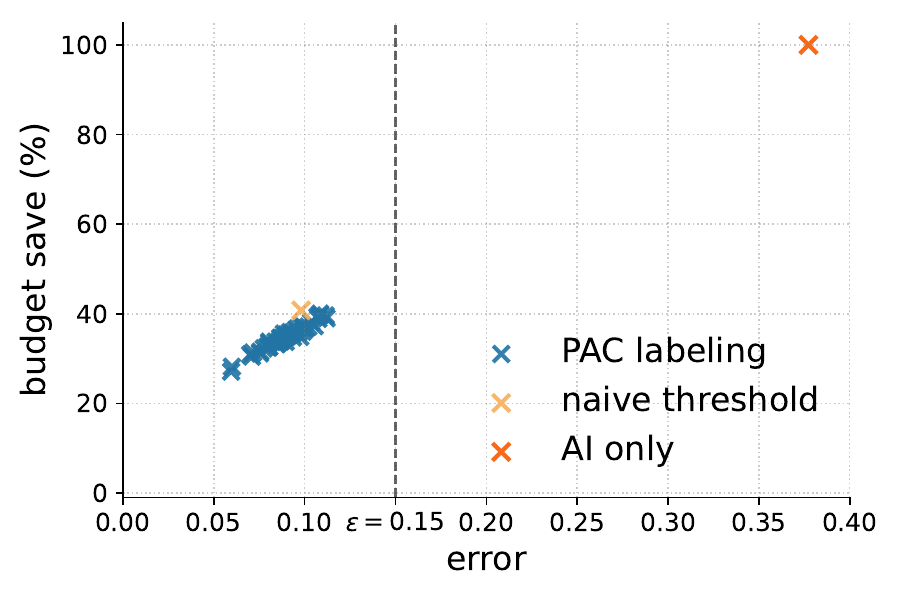}
\includegraphics[width = 0.8\textwidth]{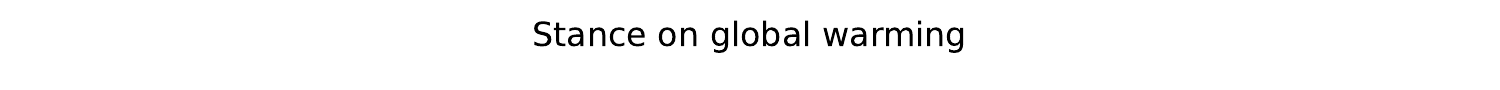}
\includegraphics[width = 0.29\textwidth]{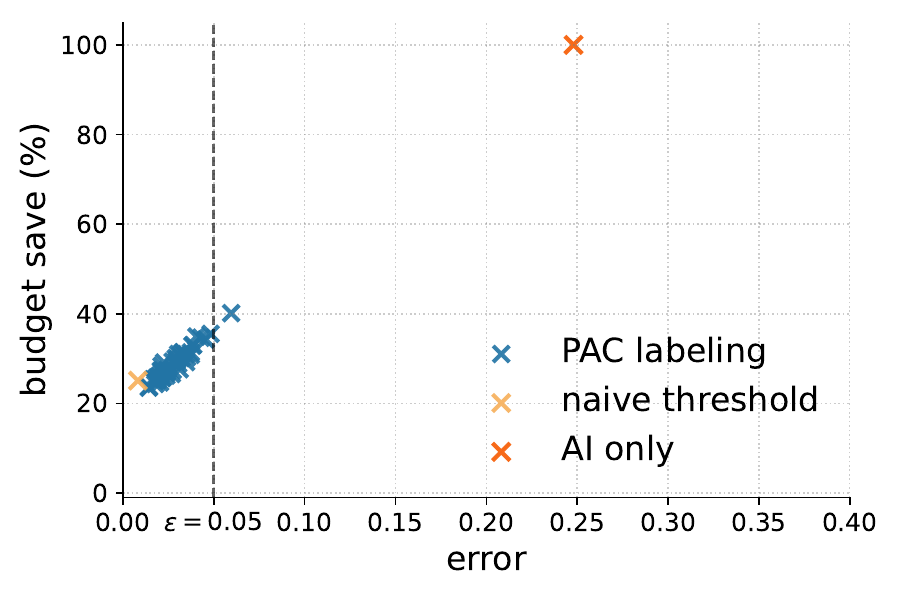}
\includegraphics[width = 0.29\textwidth]{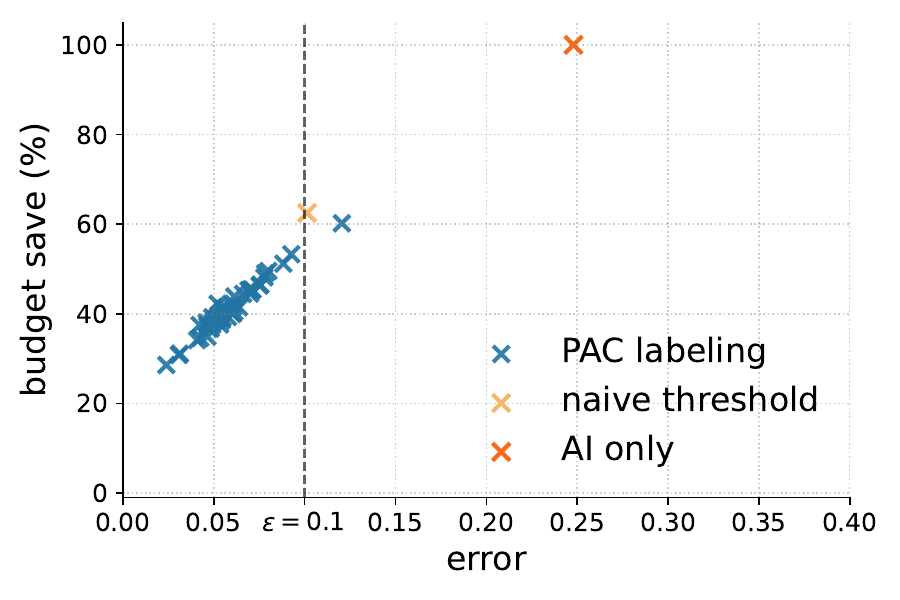}
\includegraphics[width = 0.29\textwidth]{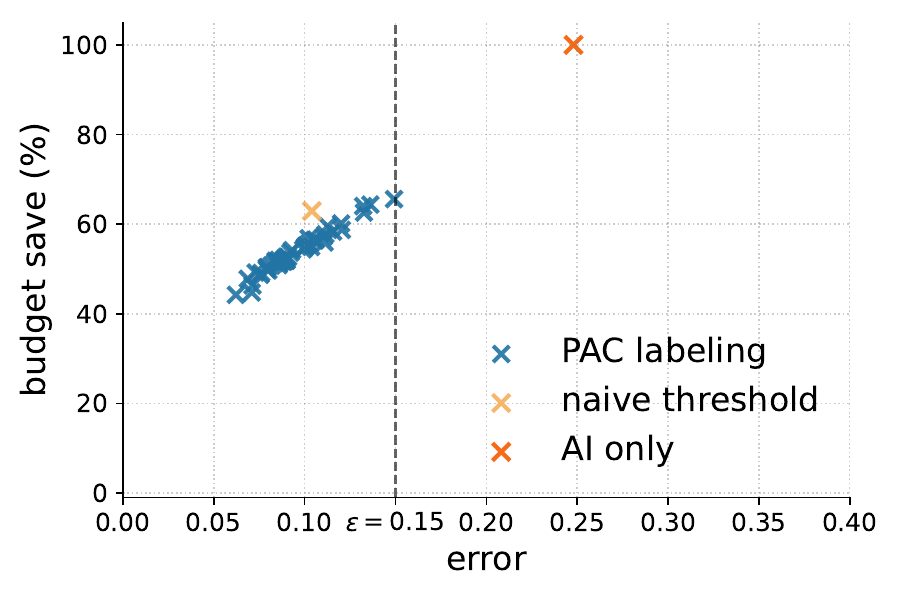}
\includegraphics[width = 0.8\textwidth]{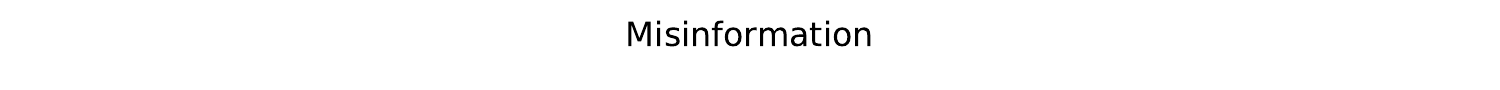}
\includegraphics[width = 0.29\textwidth]{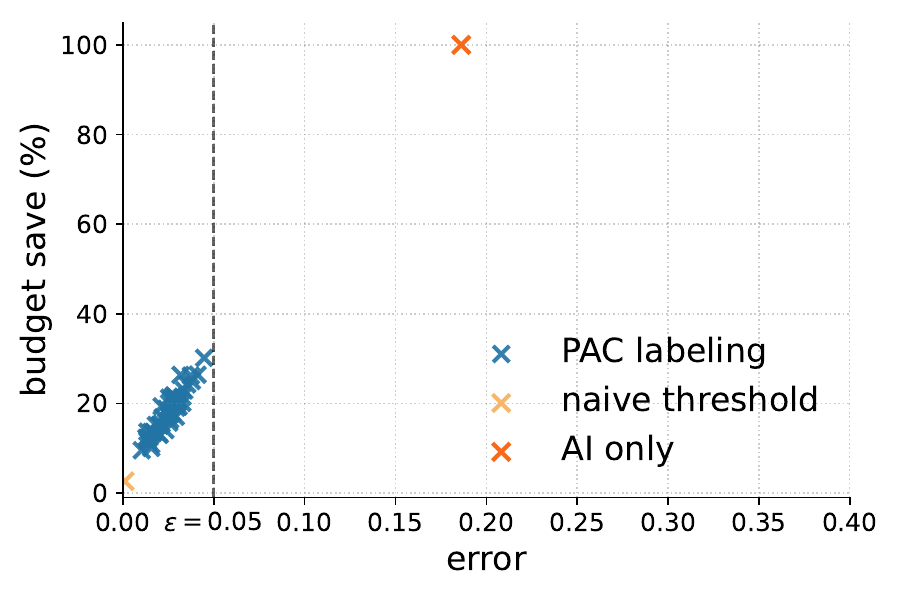}
\includegraphics[width = 0.29\textwidth]{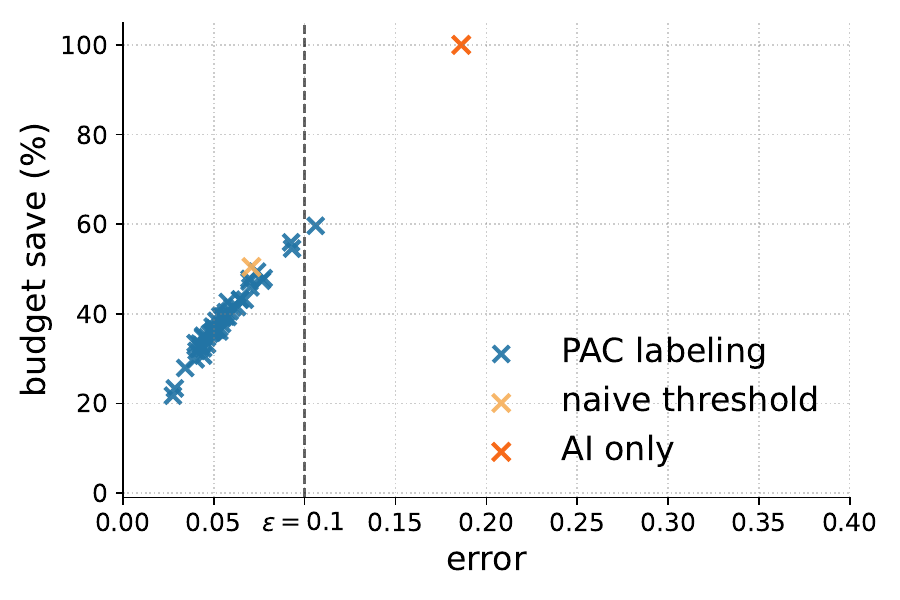}
\includegraphics[width = 0.29\textwidth]{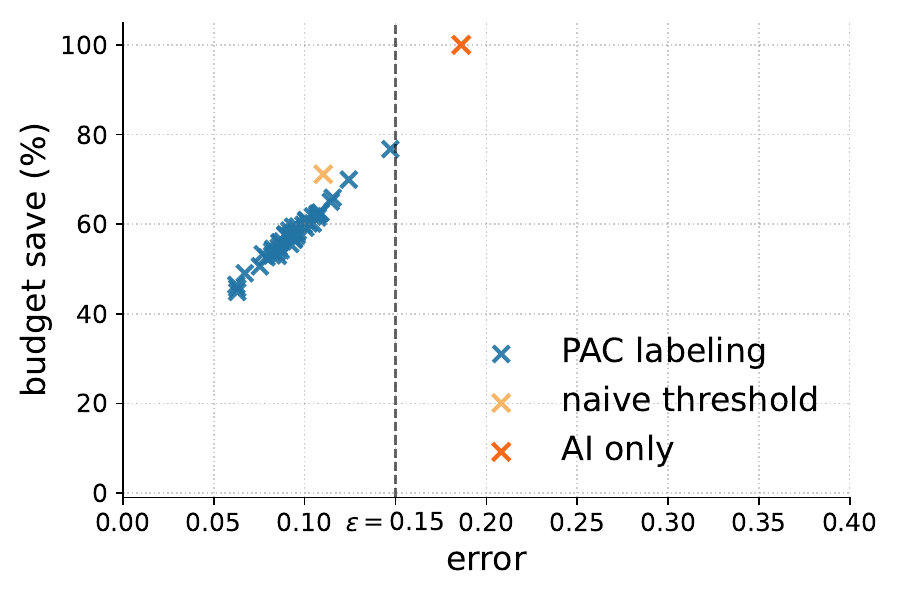}
\includegraphics[width = 0.8\textwidth]{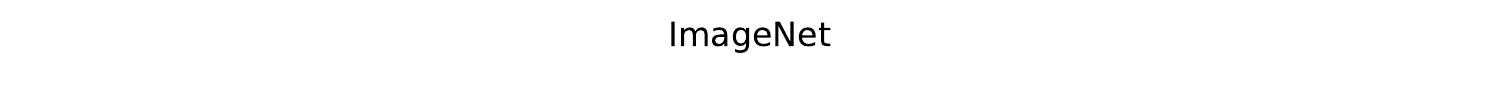}
\includegraphics[width = 0.29\textwidth]{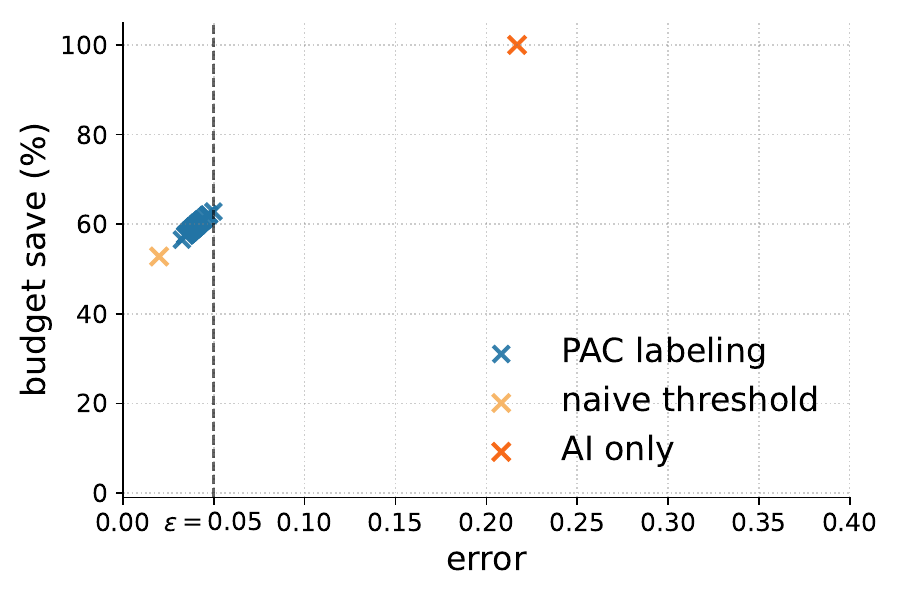}
\includegraphics[width = 0.29\textwidth]{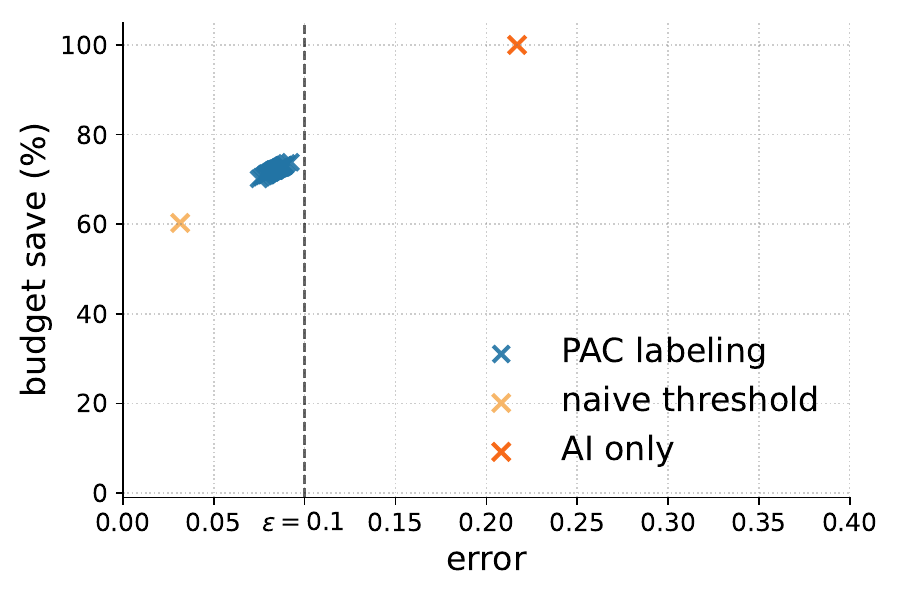}
\includegraphics[width = 0.29\textwidth]{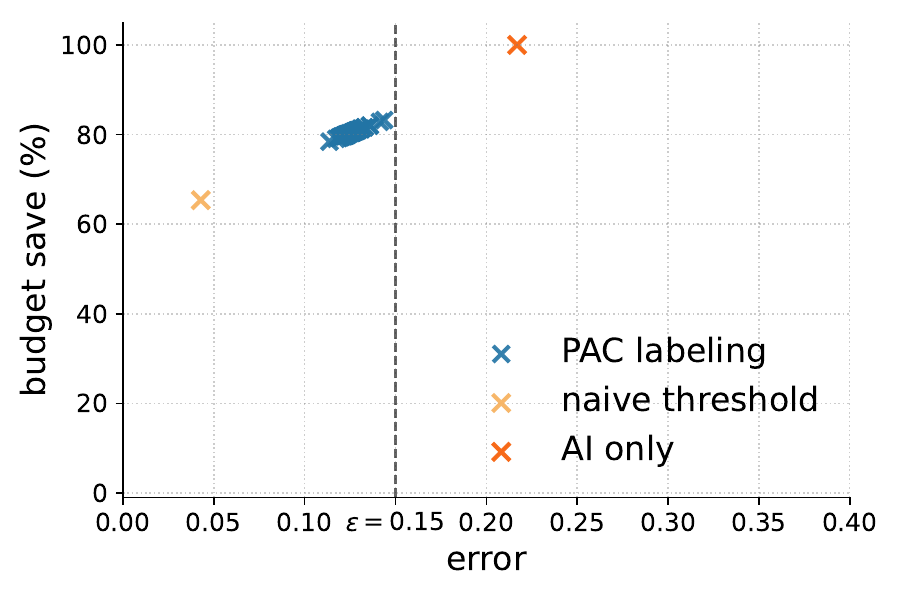}
\includegraphics[width = 0.8\textwidth]{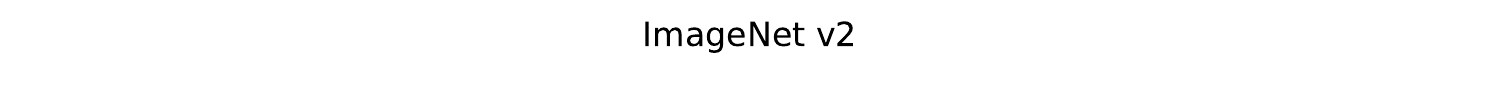}
\includegraphics[width = 0.29\textwidth]{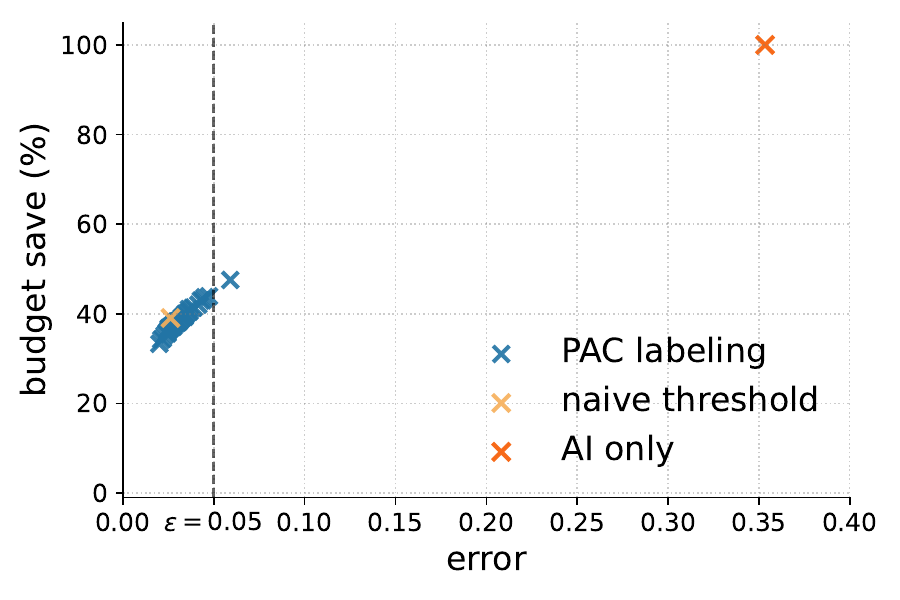}
\includegraphics[width = 0.29\textwidth]{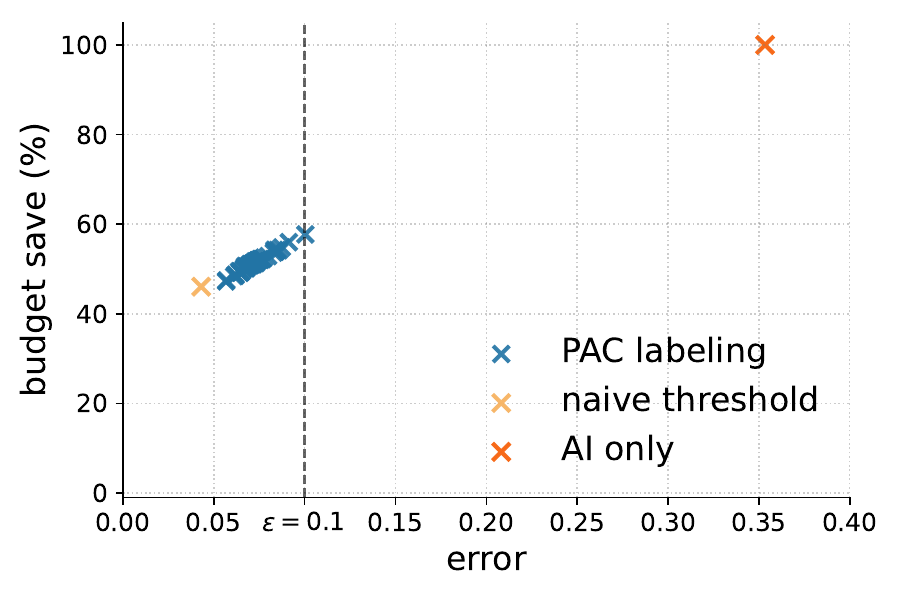}
\includegraphics[width = 0.29\textwidth]{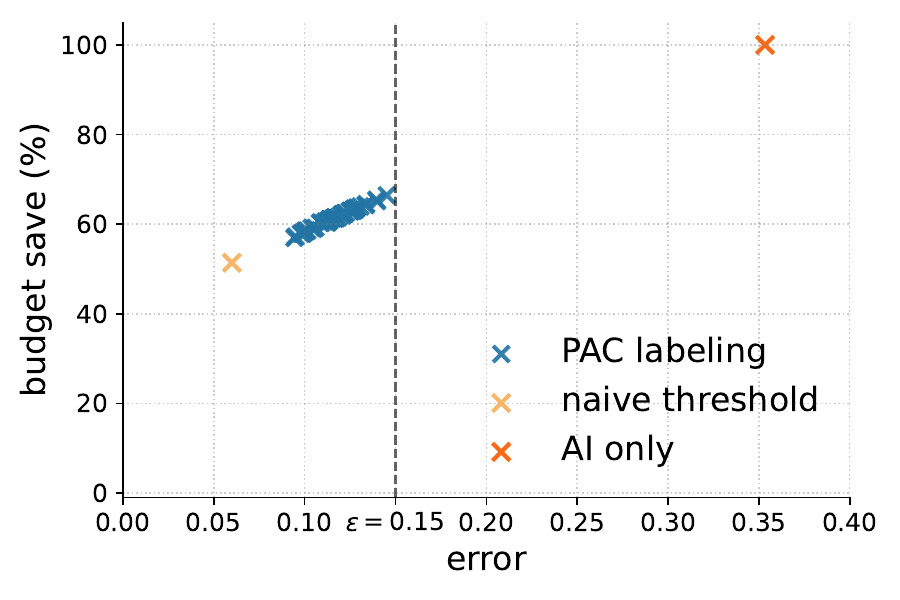}
\caption{\textbf{PAC labeling for discrete labels.} Realized error and save in budget for PAC labeling, the naive thresholding baseline, and the AI only baseline. Each row and column correspond to a different dataset and value of $\epsilon$ (denoted by vertical dashed line), respectively. For PAC labeling, we plot the realized error and save in budget for $50$ randomly chosen trials. For the naive thresholding baseline, we collect expert labels for all points with $U_i \geq \epsilon$.}
\label{fig:classification}
\end{figure}

\clearpage

\paragraph{Continuous labels.} By choosing the appropriate loss, PAC labeling is applicable much more generally. We consider two tasks.
The first is sentiment analysis~\cite{socher2013recursive}. The goal is to provide a real-valued sentiment score $Y_i\in[0,1]$ of a phrase, higher indicating more a positive sentiment. We use the squared loss, $\ell(Y_i, \tilde Y_i)=(Y_i - \tilde Y_i)^2$. We use GPT-4o to collect predicted labels $\hat Y_i$ and uncertainties $U_i$. In particular, we prompt GPT to predict an interval $[a_i,b_i]$ for the label $Y_i$, and we set $\hat Y_i = \frac{a_i+b_i}{2}$ and use the length of the interval as the uncertainty score, $U_i = b_i-a_i$. The second task is protein structure prediction \cite{jumper2021highly}. Here, $Y_i$ are experimentally derived structures and $\hat Y_i$ are structures predicted via AlphaFold \cite{jumper2021highly}. We use the mean squared deviation (MSD), the standard measure of protein structure quality, as the loss $\ell$. For context, two experimental structures for the \emph{same} protein have a gap of around $0.36$ in terms of MSD.  For the uncertainties $U_i$, we use the average predicted local distance difference test (pLDDT), AlphaFold's internal measure of local confidence. We use the CLT upper bound \eqref{eq:CLT_ub} as the mean upper bound subroutine in the algorithm.

We summarize the results in Table \ref{tab:regression} and Figure \ref{fig:regression}. Across varying $\epsilon$, PAC labeling tightly controls the error while saving a nontrivial fraction of expert labels. The AI only baseline does not meet the desired error criterion.

\begin{table*}[t!]
\small
\centering
\begin{adjustbox}{center}
\begin{tabular}{l|l|rrrr}
\toprule
\multirow{2}{*}{\textbf{Dataset}}                & \multirow{2}{*}{\textbf{Metric}}                   & \multicolumn{3}{c}{\textbf{Method}}   \\
&  & PAC ($\epsilon=0.005$) & PAC ($\epsilon=0.01$) & PAC ($\epsilon=0.015$) 
& AI only \\
\midrule
\multirow{2}{*}{\textbf{{\shortstack[l]{Sentiment analysis}}}} 
& {Budget save (\%)} & (16.03 $\pm$ 2.49)\% & (33.25 $\pm$ 3.47)\% &  (50.86 $\pm$ 3.93)\% &
---  \\
& {Error} & \hlprimarytab{0.004} & \hlprimarytab{0.009} & \hlprimarytab{0.013} 
& \hlsecondarytab{0.021} \\
\midrule
&  & PAC ($\epsilon=0.36$) & PAC ($\epsilon=0.64$) & PAC ($\epsilon=1.0$) & AI only \\
\midrule

\multirow{2}{*}{\textbf{{\shortstack[l]{Protein folding}}}} 
& {Budget save (\%)} & (19.93 $\pm$ 1.54)\% & (26.47 $\pm$ 3.37)\% &  (33.99 $\pm$ 3.76)\% &
---  \\
& {Error} & \hlprimarytab{0.367} & \hlprimarytab{0.608} & \hlprimarytab{0.944} 
& \hlsecondarytab{3.58} \\
\bottomrule
\end{tabular}
\end{adjustbox}
\caption{\textbf{PAC labeling for continuous labels.} PAC labeling (approximately)~\hlprimarytab{meets the error criterion}, while the AI only baseline has a~\hlsecondarytab{large error}.}
\label{tab:regression}
\end{table*}

\begin{figure}[t!]
\centering
\includegraphics[width = 0.8\textwidth]{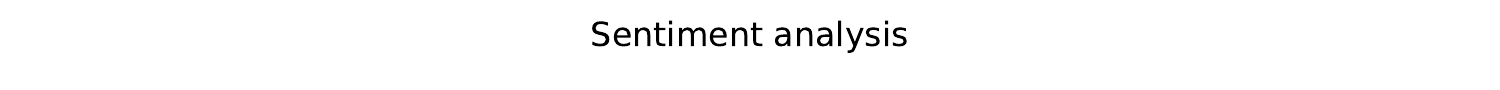}
\includegraphics[width = 0.29\textwidth]{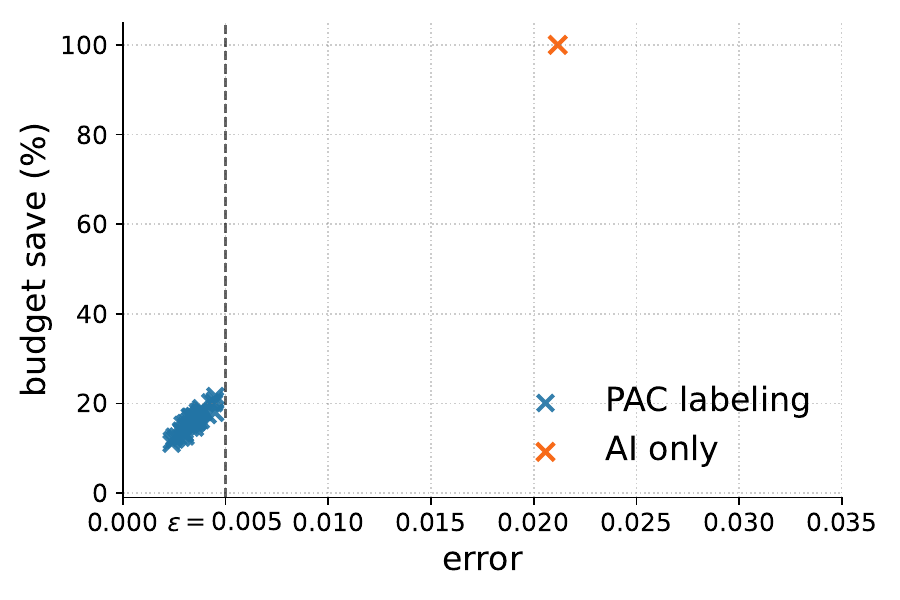}
\includegraphics[width = 0.29\textwidth]{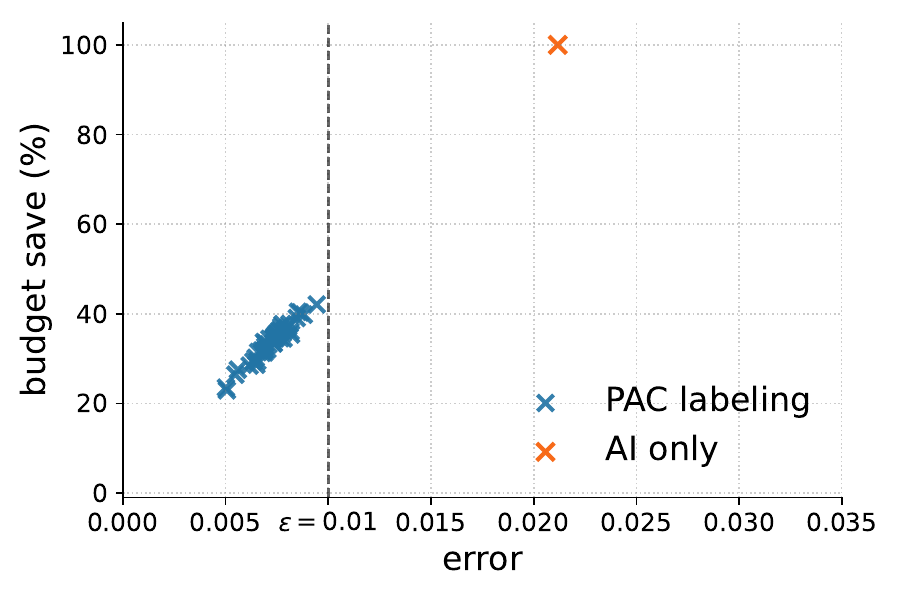}
\includegraphics[width = 0.29\textwidth]{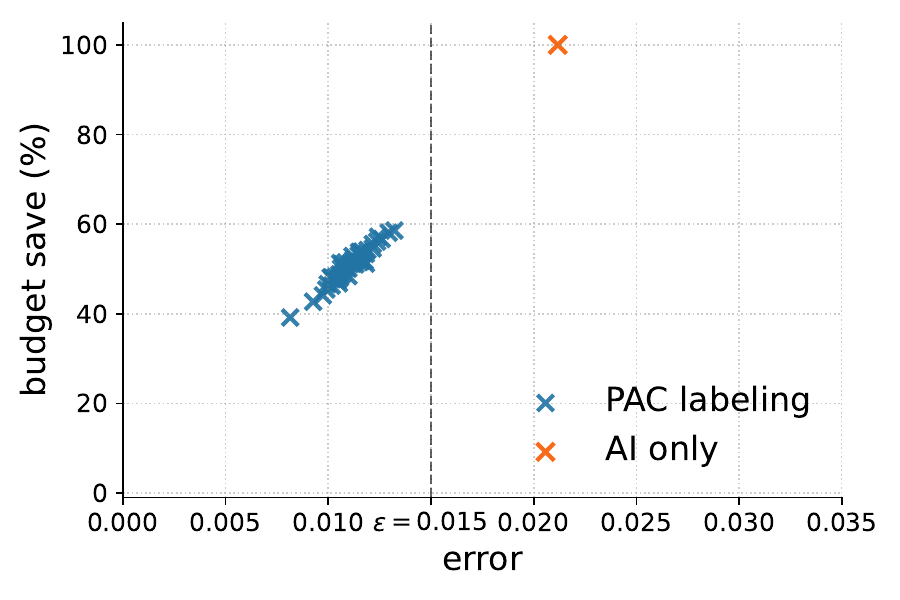}
\includegraphics[width = 0.8\textwidth]{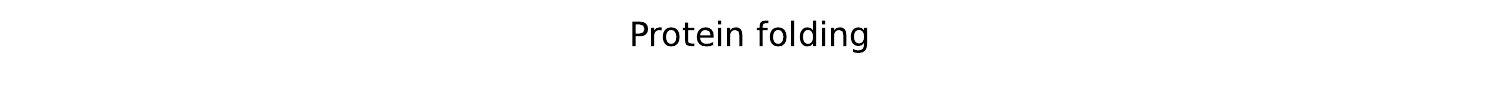}
\includegraphics[width = 0.29\textwidth]{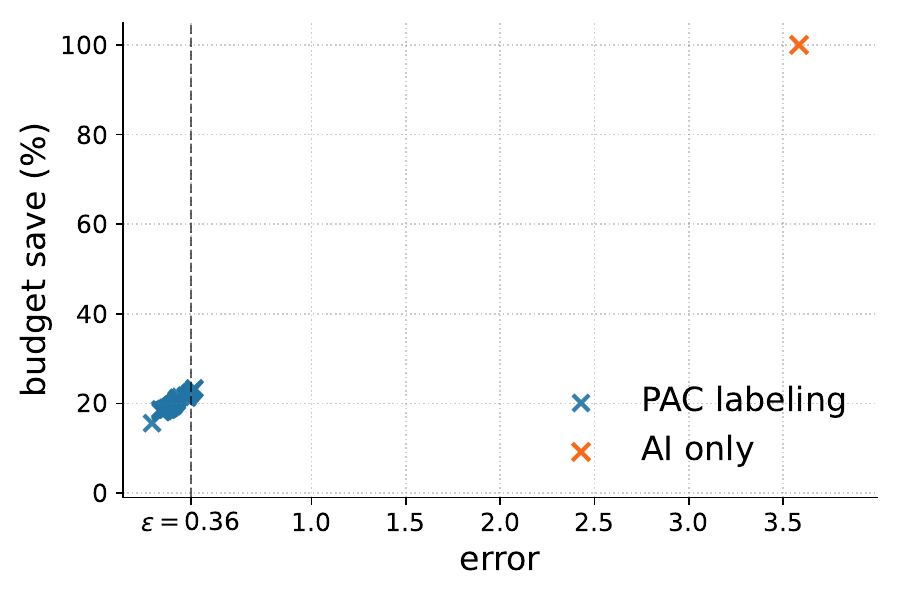}
\includegraphics[width = 0.29\textwidth]{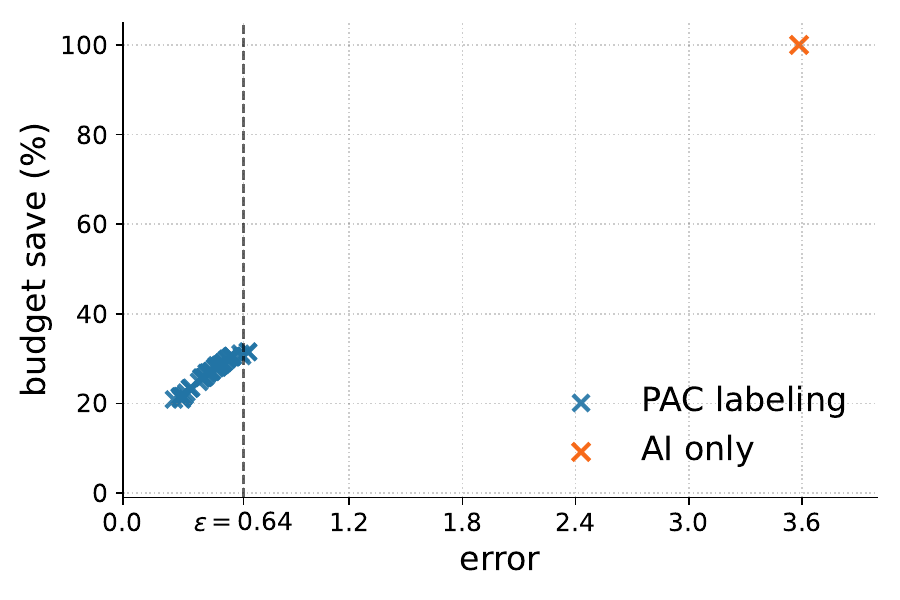}
\includegraphics[width = 0.29\textwidth]{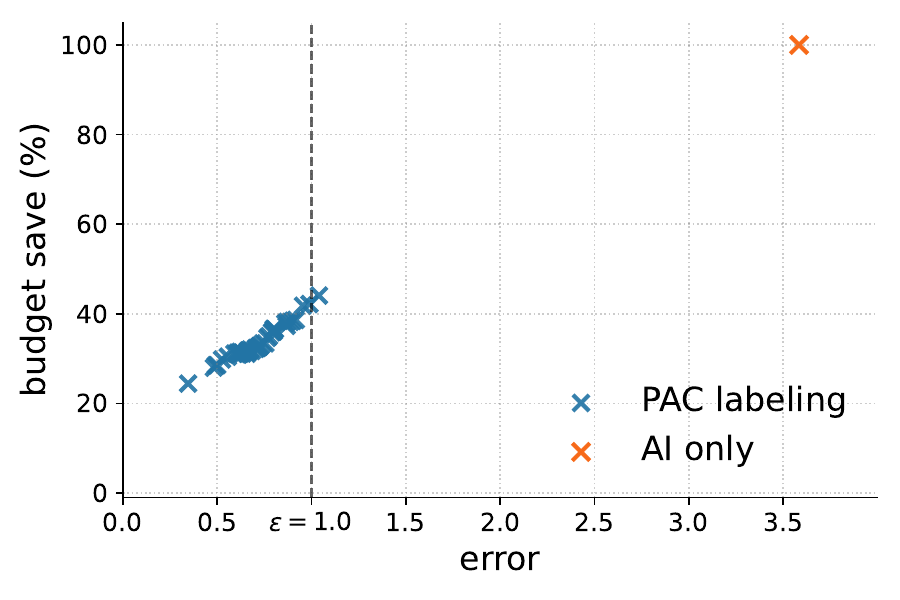}
\caption{\textbf{PAC labeling for continuous labels.} Realized error and save in budget for PAC labeling and the AI only baseline. Each row and column correspond to a different dataset and value of $\epsilon$ (denoted by vertical dashed line), respectively. For PAC labeling, we plot the realized error and save in budget for $50$ randomly chosen trials.}
\label{fig:regression}
\end{figure}

\paragraph{Uncertainty calibration.} 
Calibrating uncertainties is a simple way to improve the performance of PAC labeling.
In Table \ref{tab:uncertainty_calibration}, we show the results of PAC labeling
with GPT-4o on the media bias dataset \cite{baly2020we}, with and without
uncertainty calibration. Recall that each 
entry in this dataset corresponds to a 
news article, with true labels $Y_i \in 
\{\texttt{left}, \texttt{center}, \texttt{right}\}$ indicating the political bias of the article; 
predicted labels $\smash{\hat{Y}_i}$ capturing GPT-4o's
estimate of the label based on the article contents;
and corresponding GPT-4o uncertainties $U_i$.
We use a very simple calibration procedure: 
we use GPT-4o to cluster the articles into five clusters based on how conservative/liberal their source
(e.g., CNN, Fox News, NYT, etc.) is,
and we treat each article's cluster assignment as a
group label $G_i$.
Selecting the number of bins $B = 3$, we iterate through each group and uncertainty bin and additively adjust the uncertainties 
to match the average correctness using a small calibration set, as described in Section \ref{sec:uncertainty_calibration}.
Even in this simple setting (where the group labels 
are disjoint and derived only from the article source), 
calibration leads to an improvement in the budget save.

\begin{table*}[t!]
    \small
    \centering
    \begin{adjustbox}{center}
 \begin{tabular}{l|l|rr}
\toprule
\multirow{2}{*}{\textbf{Dataset}} & \multirow{2}{*}{\textbf{Metric}} & \multicolumn{2}{c}{\textbf{Method}} \\
& & PAC (before calibration) & PAC (after calibration) \\
\midrule
\multirow{2}{*}{\textbf{Media bias}} 
& Budget save (\%) & (13.68 $\pm$ 3.19)\% & \negthickspace (16.72 $\pm$ 2.81)\% \\
& Error & \hlprimarytab{4.10\%} & \hlprimarytab{4.22\%} \\
\bottomrule
\end{tabular}
    \end{adjustbox}
    \caption{\textbf{Uncertainty calibration.} We set $\epsilon=0.05$. PAC labeling with calibrated uncertainties (right) leads to higher saves than PAC labeling without calibration (left).
    In either case, PAC labeling~\hlprimarytab{meets the error criterion}.
    }
    \label{tab:uncertainty_calibration}
\end{table*}

\begin{table*}[t!]
\small
\centering
\begin{adjustbox}{center}
\begin{tabular}{l|l|rrr}
\toprule
\multirow{2}{*}{\textbf{Dataset}} & \multirow{2}{*}{\textbf{Metric}} & \multicolumn{3}{c}{\textbf{Method}} \\
& & PAC labeling (GPT-4o) & PAC labeling (Claude Sonnet) & PAC router \\
\midrule
\multirow{4}{*}{\textbf{Media bias}} 
& Budget save (\%) & (13.79 $\pm$ 3.38)\% & \negthickspace (8.41  $\pm$ 3.01)\% & \negthickspace (41.61  $\pm$ 1.50)\% \\
& Error & \hlprimarytab{4.10\%} & \hlprimarytab{4.00\%} & \hlprimarytab{4.61\%} \\
\cmidrule(lr){2-5}
& Save in cost & (188.66 $\pm$ 41.15)\% & (131.36 $\pm$ 49.20)\% & (482.04 $\pm$ 114.73)\% \\
& Error & \hlprimarytab{4.06\%} & \hlprimarytab{3.58\%} & \hlprimarytab{3.61\%} \\
\bottomrule
\end{tabular}
\end{adjustbox}
\caption{\textbf{PAC router for language models.} We set $\epsilon=0.05$. The PAC router significantly improves the budget save (top) and save in cost (bottom) compared to PAC labeling with individual models. In all cases, PAC labeling~\hlprimarytab{meets the error criterion}.}
\label{tab:router}
\end{table*}

\subsection{PAC labeling with multiple models}

Next, we consider the multi-model case. 
We revisit the problem of annotating the political bias of media articles \cite{baly2020we}. In addition to GPT-4o predictions and confidences, we also collect predictions and confidences from Claude 3.7 Sonnet. We train a PAC router, as described in Section \ref{sec:router}, to route the articles between the two language models. We simultaneously train an uncertainty model, as discussed in Section \ref{sec:router_learned_uncertainties}. We again use the betting confidence intervals \cite{waudby2020estimating} as the mean upper bound subroutine.

\paragraph{Costless predictions.} 
First we consider the setting of costless predictions, aiming only to minimize the number of collected expert labels. See Figure \ref{fig:router} (top) and Table \ref{tab:router} (top) for the results. GPT and Claude alone allow for a roughly $14\%$ and $8\%$ budget save, respectively, while by routing between the two we can save about $42\%$ of the expert label cost.

To give further intuition behind how this gain is achieved, in Figure \ref{fig:Lu} we plot the loss $L^u = \frac{1}{n} \sum_{i=1}^n \ell^u(Y_i,\hat Y_i)$ that results from collecting labels at uncertainties greater than or equal to $u$, as a function of $u$. To account for the fact that the different baselines might gives uncertainties $U_i$ of different magnitudes, without loss of generality we first map the uncertainties to their respective rank in $\{1,\dots,n\}$.
We observe that the router produces a curve $L^u$ that strictly dominates the loss curves of the individual models. This means that, for any uncertainty threshold, the resulting labeling achieves a strictly smaller error than with a single model. As a result, the critical uncertainty at which $L^u$ crosses error $\epsilon$ is significantly larger.

\paragraph{Incorporating prediction costs.}
We also consider the cost-sensitive
setting, where we take into account the costs of GPT-4o and Claude 3.7 Sonnet
labels and aim to minimize the overall labeling cost. 
We use the true current
relative costs of the two models. We set $c_{\mathrm{expert}} =1$,
$c_{\mathrm{GPT}} = 0.25$, and $c_{\mathrm{Claude}} = 0.075$. 
We show the
results in Figure \ref{fig:router} (bottom) and Table \ref{tab:router} (bottom):
cost-sensitive routing more than doubles the save in cost compared to GPT and more than triples the save compared to Claude.

\begin{figure}[t]
\centering
\includegraphics[width = 0.8\textwidth]{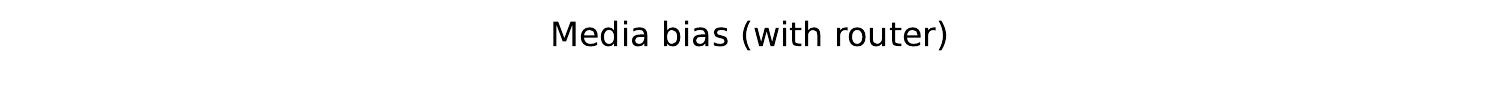}
\includegraphics[width = 0.29\textwidth]{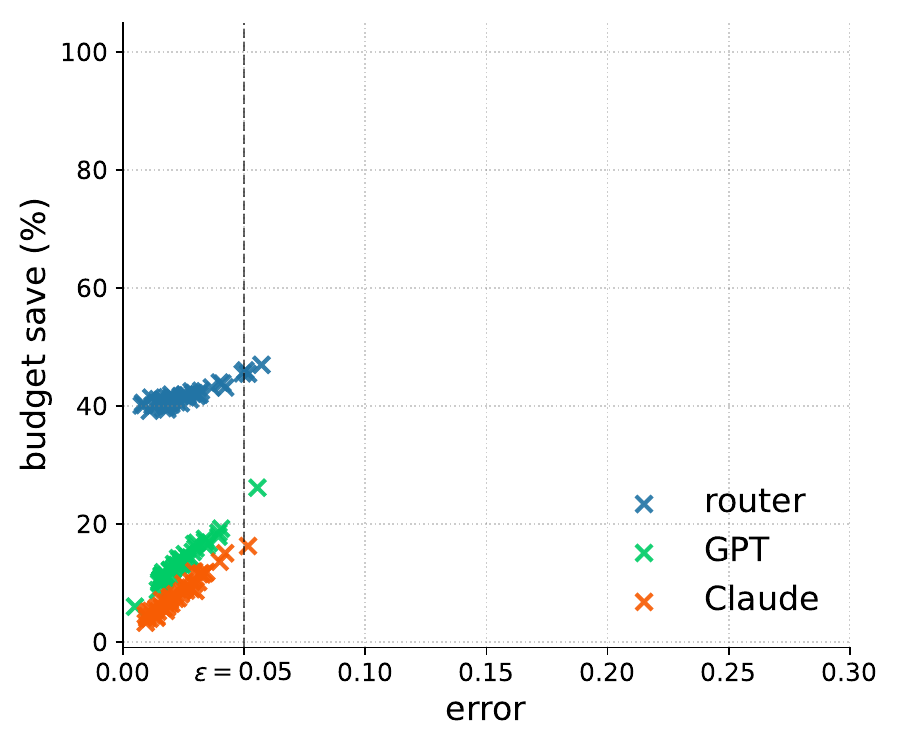}
\includegraphics[width = 0.29\textwidth]{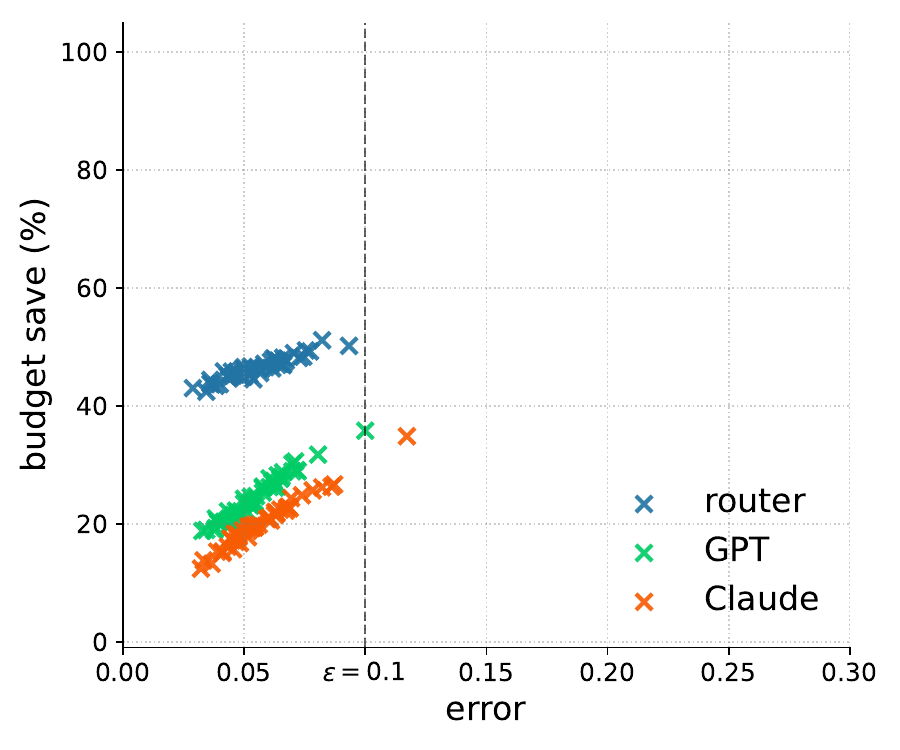}
\includegraphics[width = 0.29\textwidth]{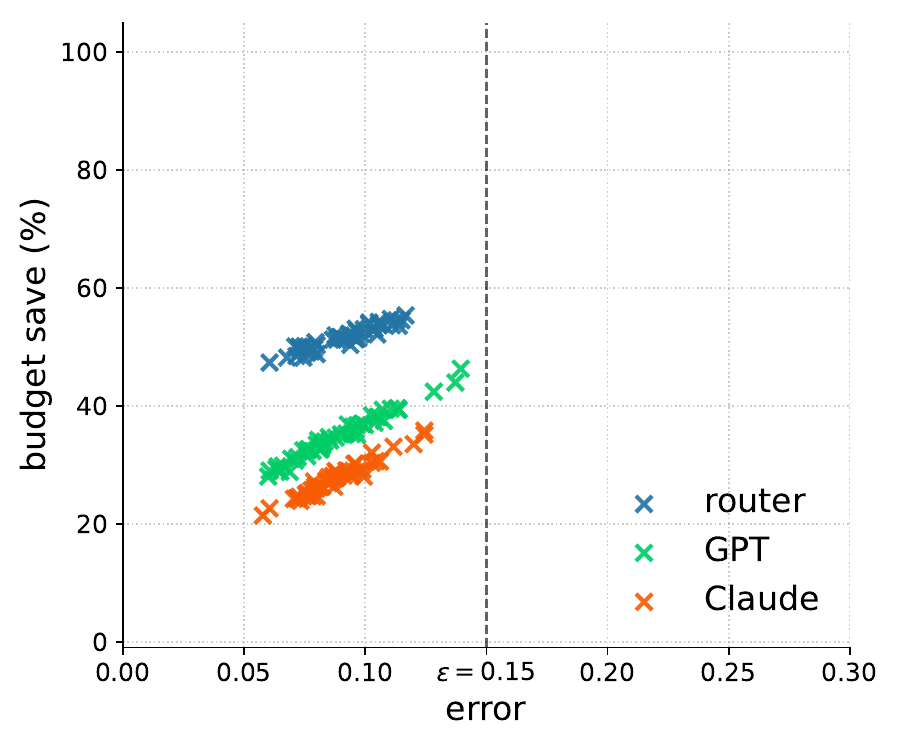}
\includegraphics[width = 0.29\textwidth]{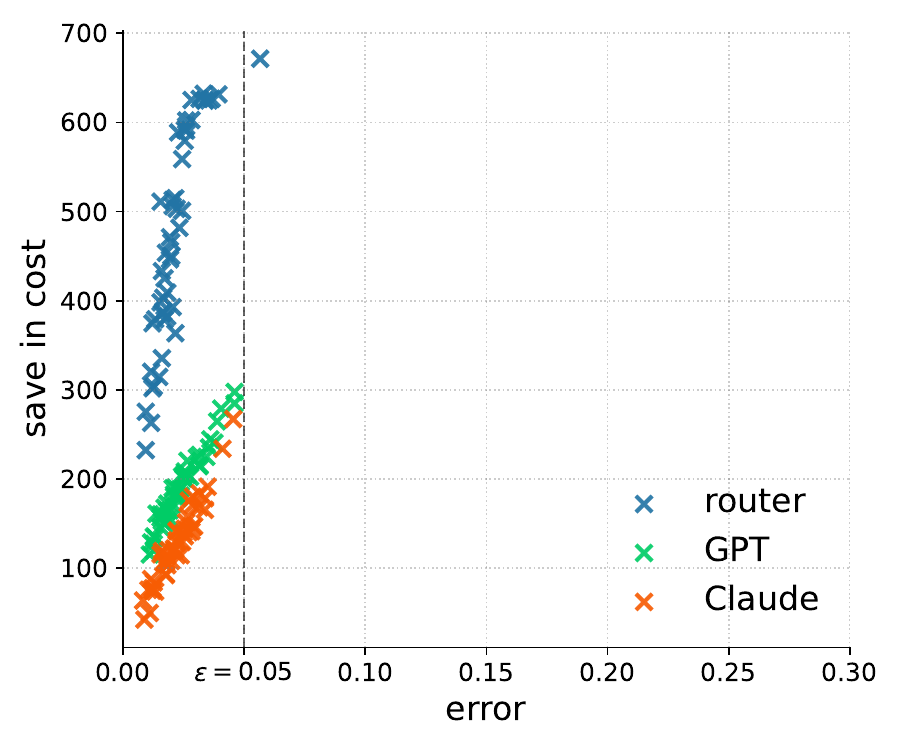}
\includegraphics[width = 0.29\textwidth]{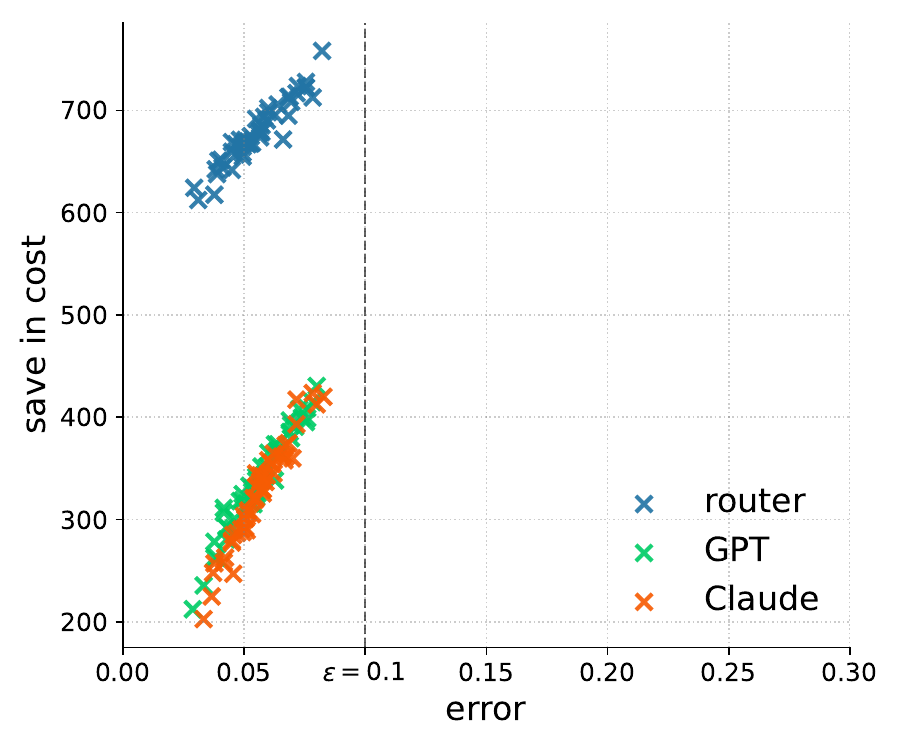}
\includegraphics[width = 0.29\textwidth]{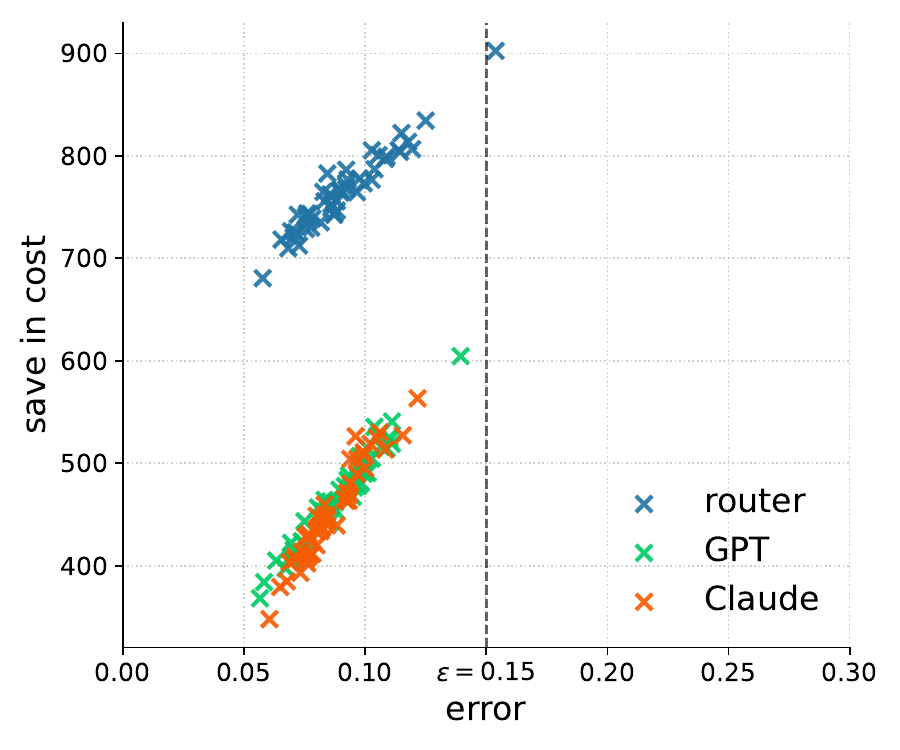}
\caption{\textbf{PAC router for language models.} Realized error and save in budget for PAC labeling with GPT, PAC labeling with Claude, and the PAC router between GPT and Claude. The top row corresponds to the costless setting; the bottom row corresponds to the cost-sensitive setting.
Each column corresponds to a different value of $\epsilon$ (denoted by vertical dashed line). For each method, we plot the realized error and save in budget for $50$ randomly chosen trials.}
\label{fig:router}
\end{figure}

\begin{figure}[h]
\centering
\includegraphics[width = 0.342\textwidth]{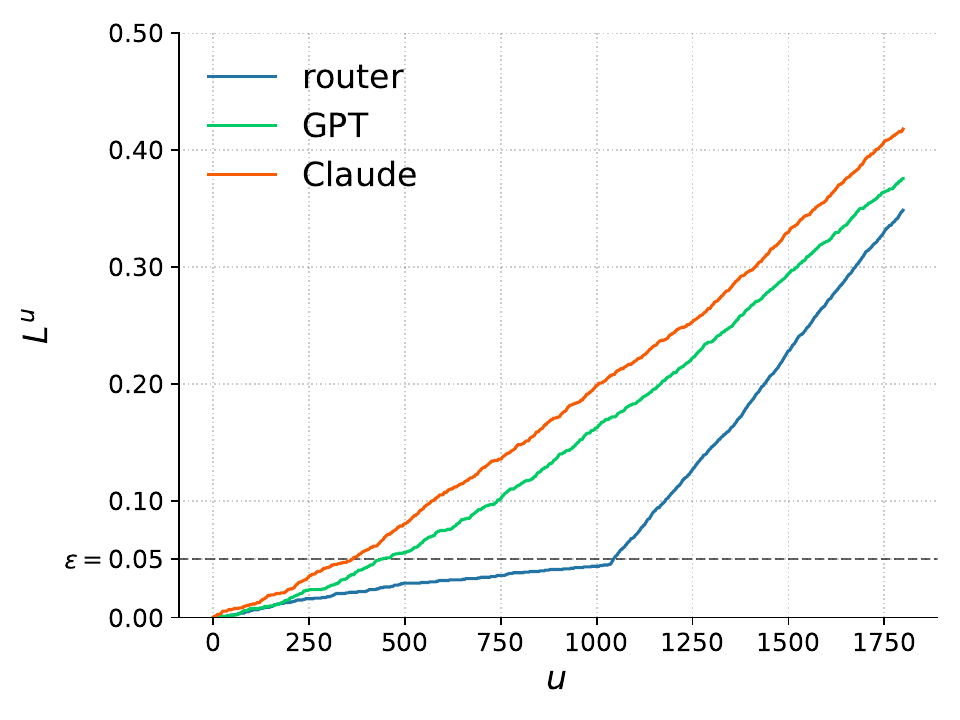}
\caption{\textbf{Loss $L^u$ after PAC routing.} Error $L^u$ after collecting labels at uncertainties greater than or equal to $u$, as a function of $u$, for GPT and Claude individually and the PAC router. We observe that the router achieves a lower error $L^u$ than the individual baselines, for all $u$.}
\label{fig:Lu}
\end{figure}

\clearpage

\bibliographystyle{plainnat}
\bibliography{refs}

\appendix

\newpage

\section{Additional results with asymptotic confidence intervals}
\label{app:asymptotic}

We include asymptotic analogues of the nonasymptotic results from Section \ref{sec:exps_single_model}. We rerun all experiments with discrete labels, this time using the asymptotic mean upper bound \eqref{eq:CLT_ub} in the construction of PAC labels.

In Table \ref{tab:classification_text_asymp} and Table \ref{tab:classification_vision_asymp} we compare PAC labeling with asymptotic and nonasymptotic guarantees on text and image datasets, respectively. We see that asymptotic confidence intervals, in addition to being easier to implement, enable larger budget saves compared to nonasymptotic intervals. The downside of relying on asymptotic guarantees is that the error rates might be slightly inflated---throughout we see error rates slightly above the nominal $5\%$. 

\begin{figure}[b]
\centering
\includegraphics[width = 0.8\textwidth]{plots/bias_title.pdf}
\includegraphics[width = 0.29\textwidth]{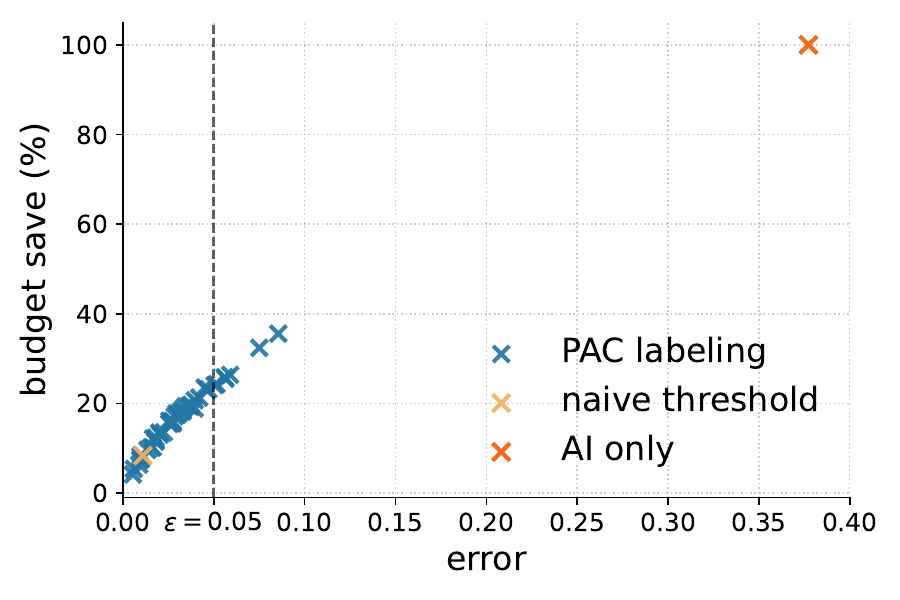}
\includegraphics[width = 0.29\textwidth]{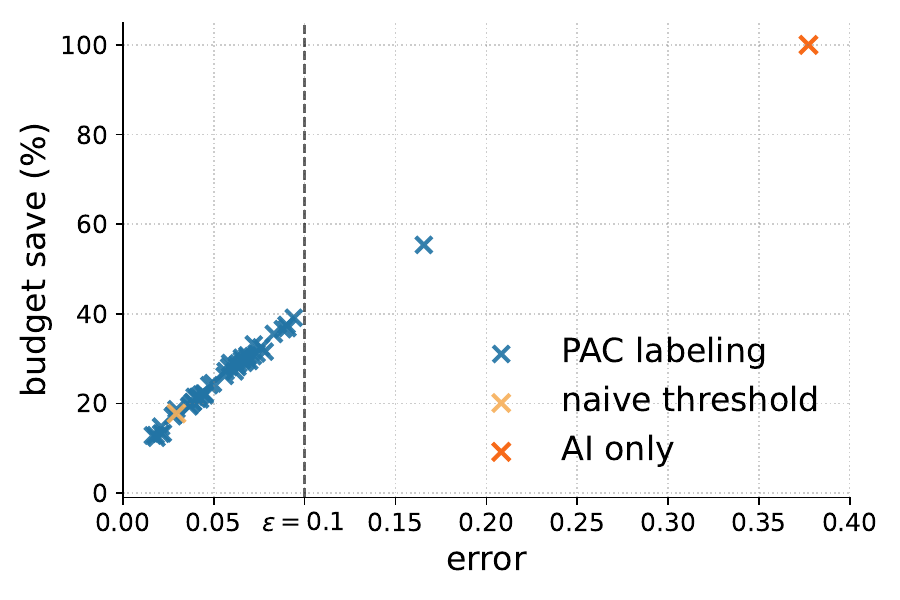}
\includegraphics[width = 0.29\textwidth]{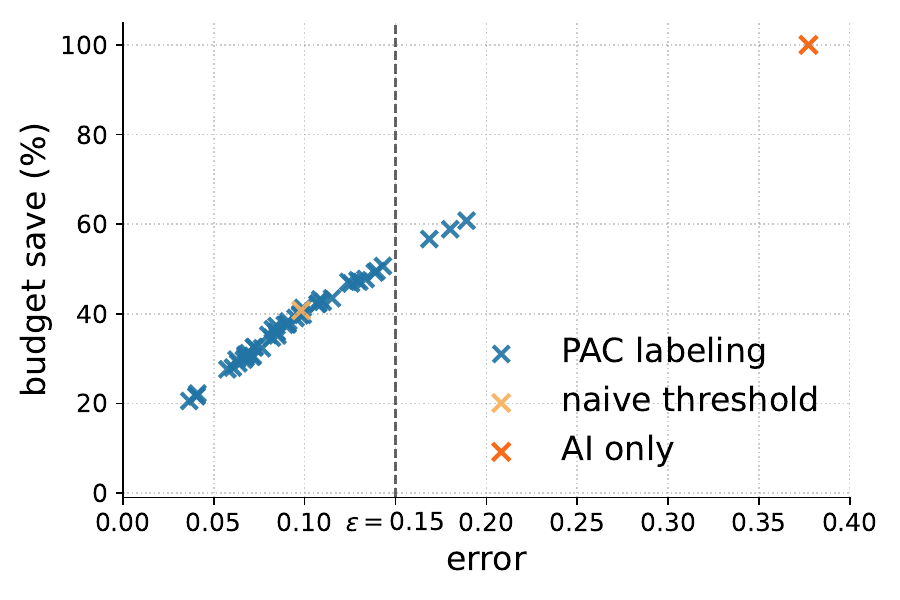}
\includegraphics[width = 0.8\textwidth]{plots/stance_title.pdf}
\includegraphics[width = 0.29\textwidth]{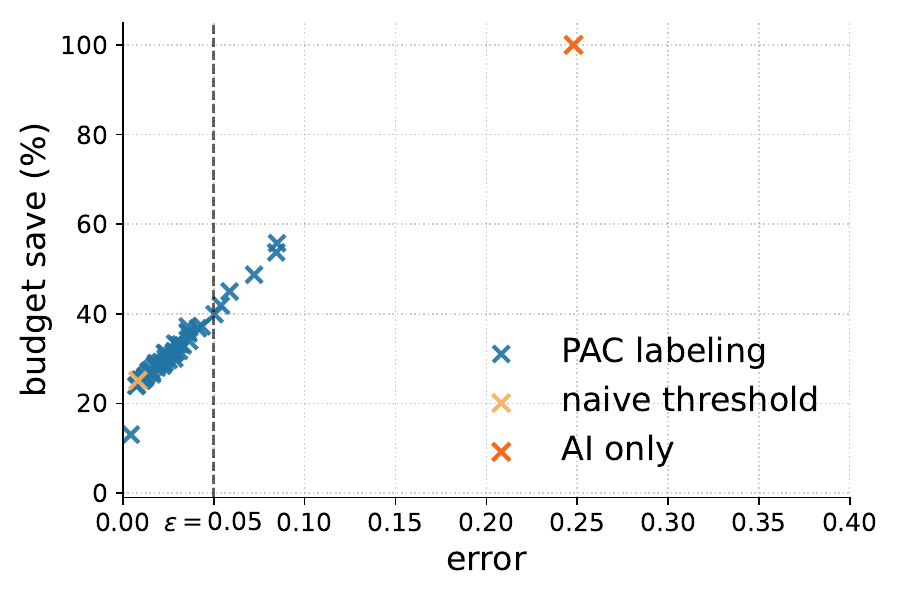}
\includegraphics[width = 0.29\textwidth]{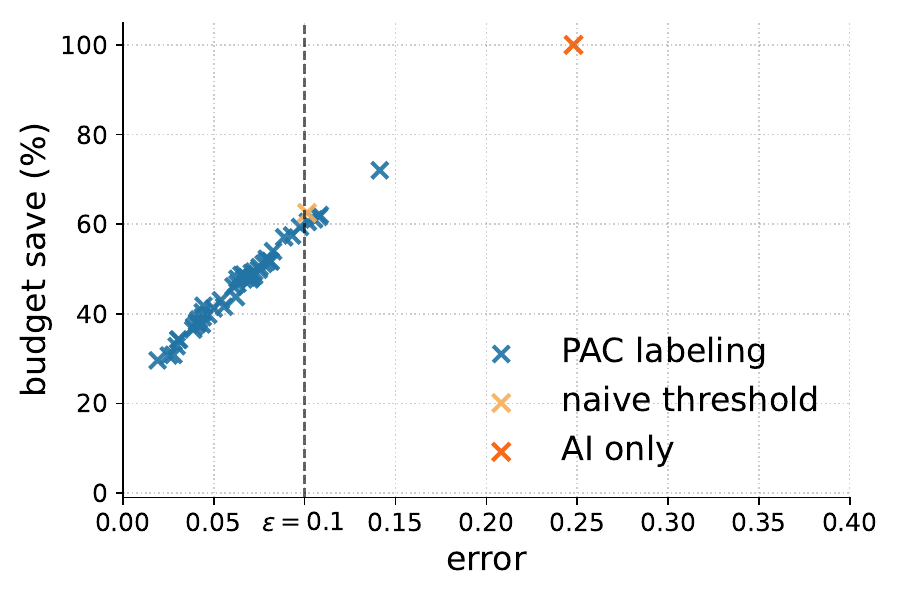}
\includegraphics[width = 0.29\textwidth]{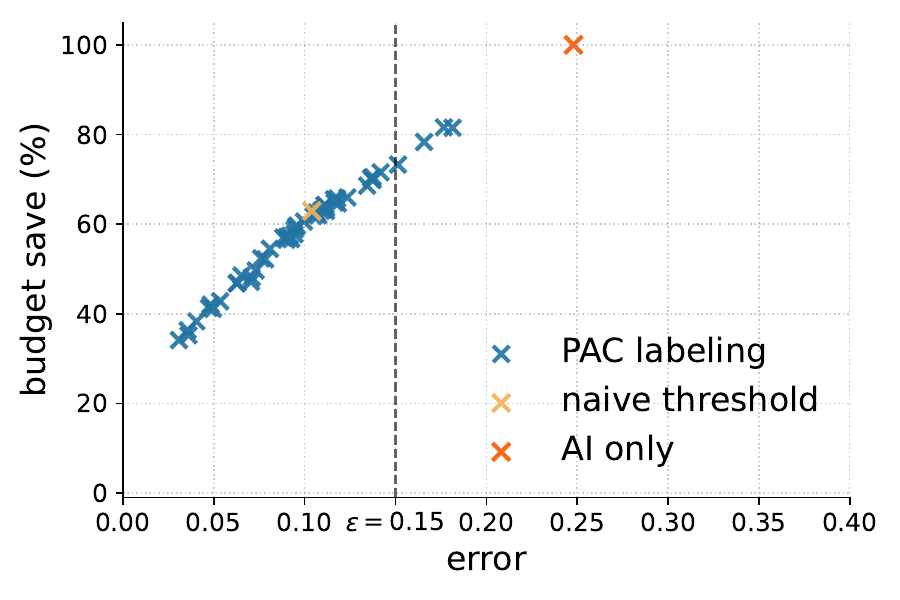}
\includegraphics[width = 0.8\textwidth]{plots/misinfo_title.pdf}
\includegraphics[width = 0.29\textwidth]{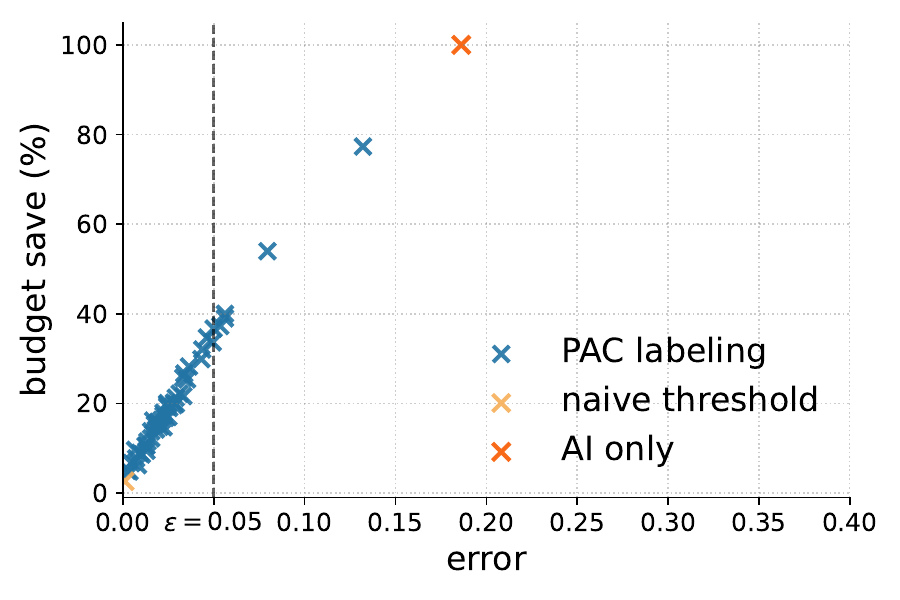}
\includegraphics[width = 0.29\textwidth]{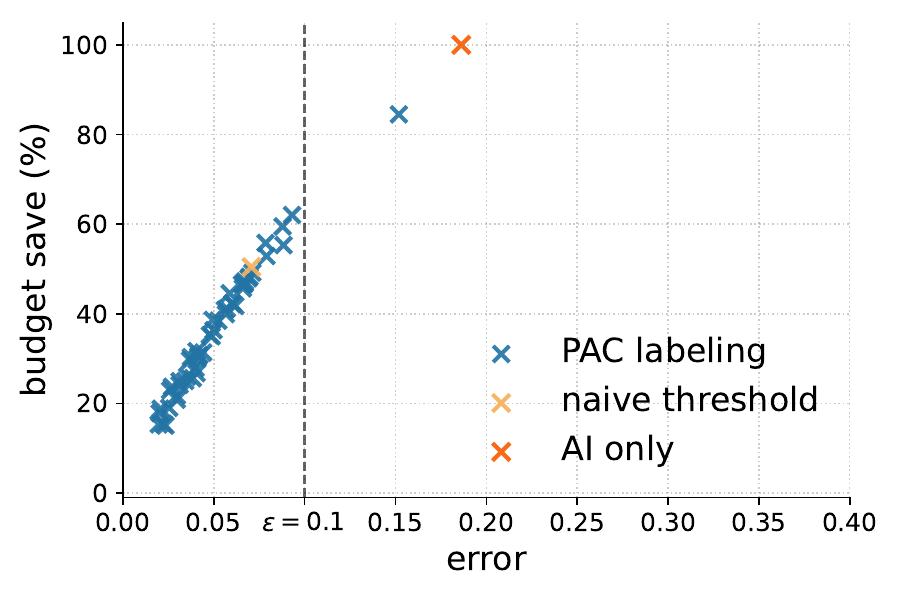}
\includegraphics[width = 0.29\textwidth]{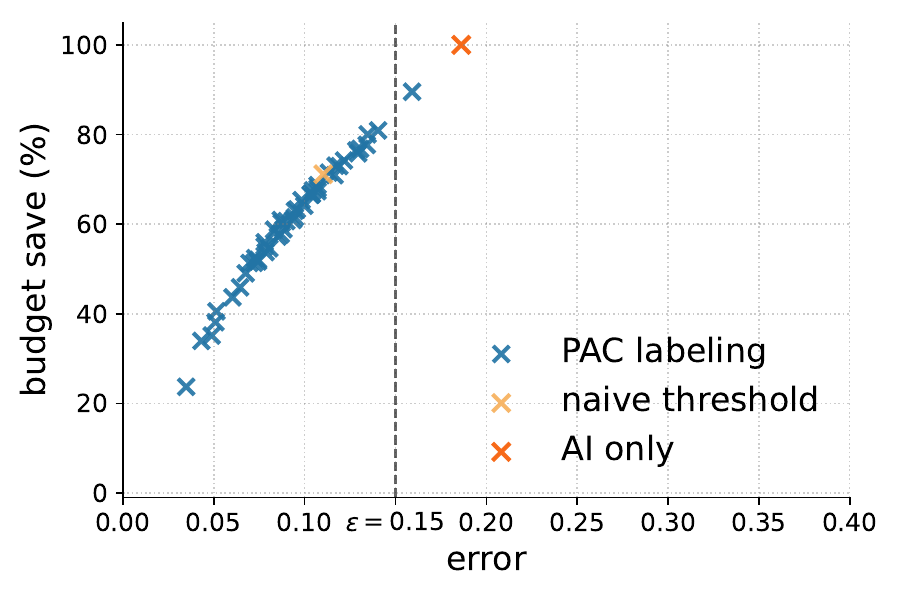}
\includegraphics[width = 0.8\textwidth]{plots/imagenet_title.pdf}
\includegraphics[width = 0.29\textwidth]{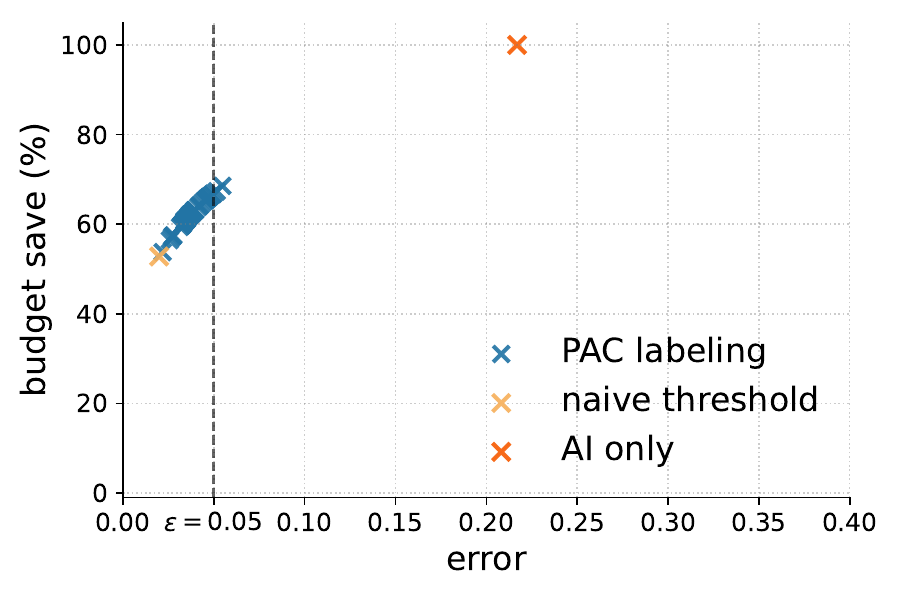}
\includegraphics[width = 0.29\textwidth]{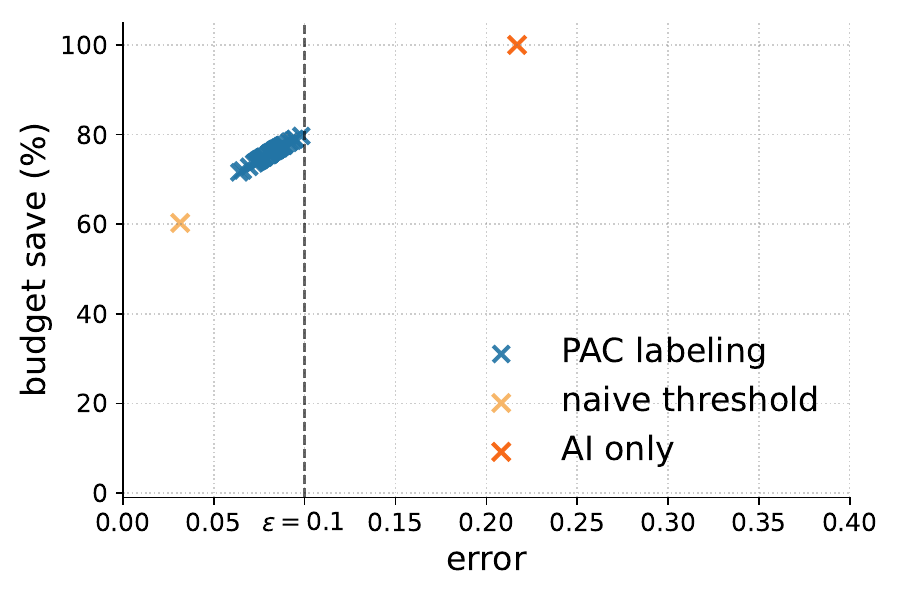}
\includegraphics[width = 0.29\textwidth]{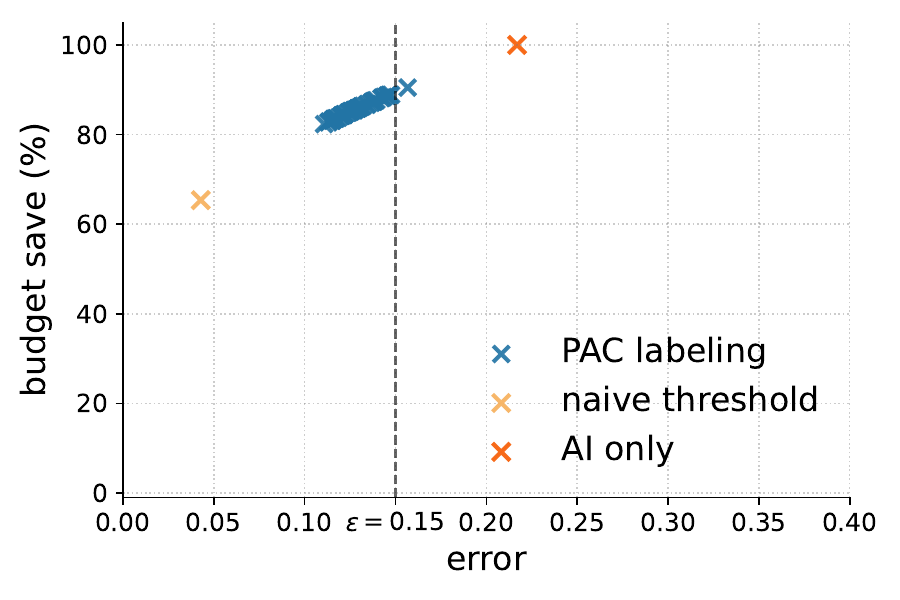}
\includegraphics[width = 0.8\textwidth]{plots/imagenetv2_title.pdf}
\includegraphics[width = 0.29\textwidth]{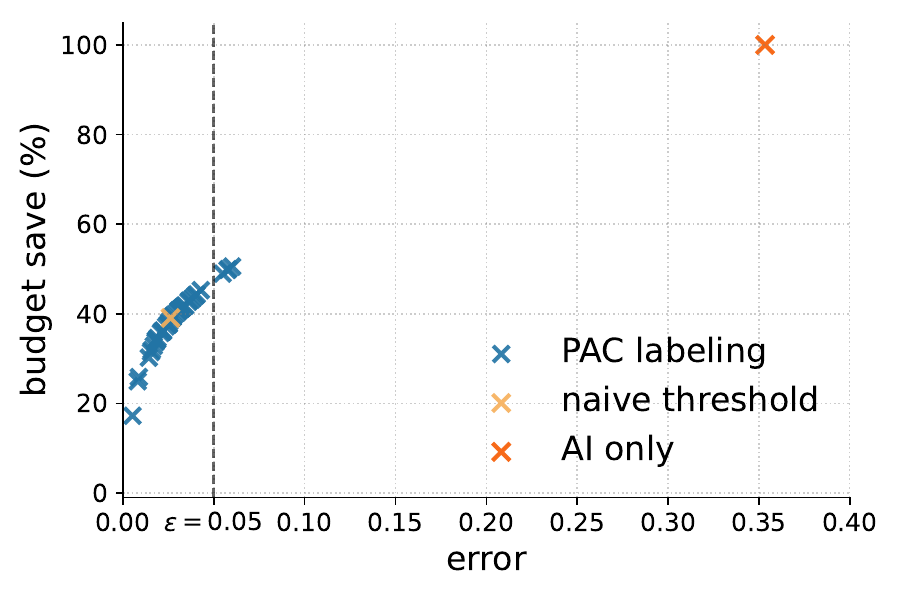}
\includegraphics[width = 0.29\textwidth]{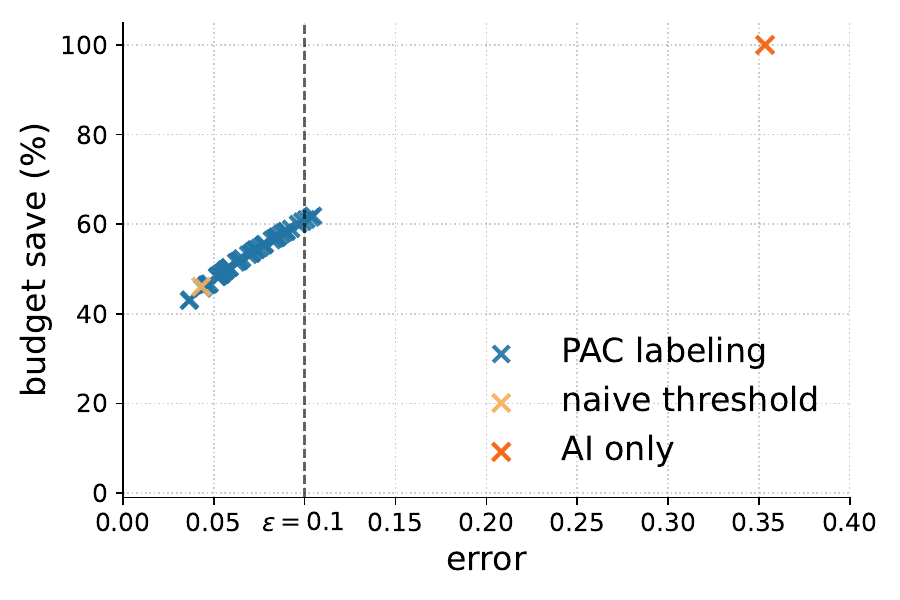}
\includegraphics[width = 0.29\textwidth]{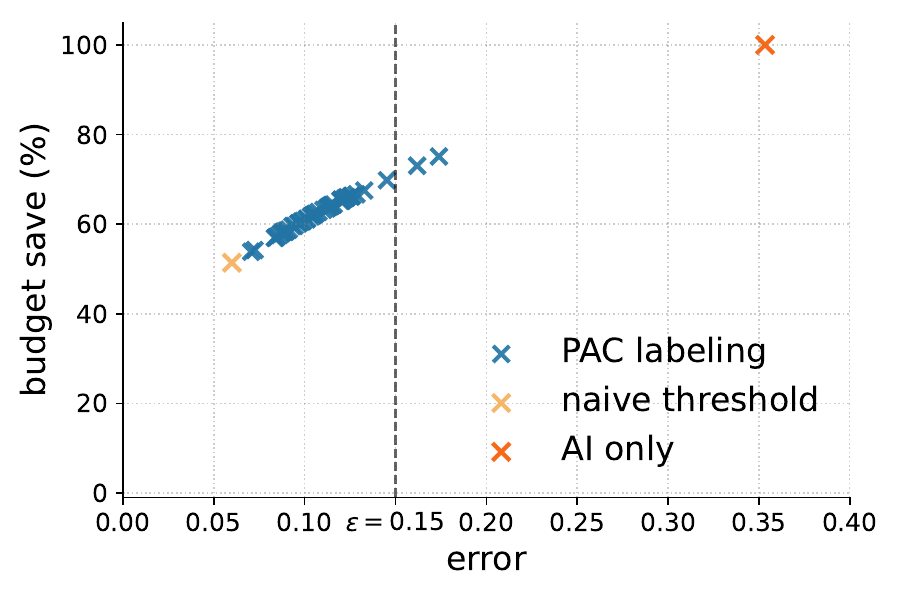}
\caption{\textbf{PAC labeling for discrete labels with asymptotic confidence intervals.} Realized error and save in budget for PAC labeling, the naive thresholding baseline, and the AI only baseline. Each row and column correspond to a different dataset and value of $\epsilon$ (denoted by vertical dashed line), respectively. For PAC labeling, we plot the realized error and save in budget for $50$ randomly chosen trials. For the naive thresholding baseline, we collect expert labels for all points with $U_i \geq \epsilon$.}
\label{fig:classification_asymp}
\end{figure}

\begin{table*}[h]
\small
\centering
\begin{adjustbox}{center}
\begin{tabular}{l|l|rr}
\toprule
\multirow{2}{*}{\textbf{Dataset}}                & \multirow{2}{*}{\textbf{Metric}}                   & \multicolumn{2}{c}{\textbf{Method}}   \\
&  & PAC labeling (asymptotic) & PAC labeling (nonasymptotic)  \\
\midrule
\multirow{2}{*}{\textbf{\shortstack[l]{Media bias }}} 
& {Budget save (\%)} & (16.11 $\pm$ 6.96)\% & (13.79 $\pm$ 3.38)\%  \\
& {Error} & 5.17\%  & 4.10\%    \\

\midrule
                               
\multirow{2}{*}{\textbf{\shortstack[l]{Stance on \\ global warming}}} 
& {Budget save (\%)} & (32.15 $\pm$ 7.38)\% & (28.09 $\pm$ 3.28)\% \\
& {Error} & 5.92\% & 4.57\%    \\

 \midrule
 \multirow{2}{*}{\textbf{{\shortstack[l]{Misinformation}}}} 
& {Budget save (\%)} & (21.41 $\pm$ 10.95)\%  & (18.12 $\pm$ 4.93)\%  \\
& {Error} & 5.83\% & 3.80\%  \\

\bottomrule
\end{tabular}
\end{adjustbox}
\caption{\textbf{PAC labeling text datasets with GPT-4o, with asymptotic (left) and nonasymptotic (right) confidence intervals.} We set $\epsilon=0.05$. PAC labeling with asymptotic guarantees enables larger saves, but may lead to slightly inflated error rates.}
\label{tab:classification_text_asymp}
\end{table*}

\begin{table*}[h]
\small
\centering
\begin{adjustbox}{center}
\begin{tabular}{l|l|rrrr}
\toprule
\multirow{2}{*}{\textbf{Dataset}}                & \multirow{2}{*}{\textbf{Metric}}                   & \multicolumn{2}{c}{\textbf{Method}}   \\
&  & PAC labeling (asymptotic) & PAC labeling (nonasymptotic)   \\
\midrule

\multirow{2}{*}{\textbf{{\shortstack[l]{ImageNet}}}} 

& {Budget save (\%)} & (62.82 $\pm$ 2.57)\% &  (59.64 $\pm$ 1.49)\%  \\
& {Error} & 5.06\% & 4.73\%\\
\midrule
                               
\multirow{2}{*}{\textbf{\shortstack[l]{ImageNet v2}}} 
& {Budget save (\%)} & 
(39.20 $\pm$ 5.82)\% & (39.07 $\pm$ 2.67)\%\\
& {Error} & 5.38\% & 4.74\%   \\

\bottomrule
\end{tabular}
\end{adjustbox}
\caption{\textbf{PAC labeling image datasets with ResNet-152, with asymptotic (left) and nonasymptotic (right) confidence intervals.} We set $\epsilon=0.05$. PAC labeling with asymptotic guarantees enables larger saves, but may lead to slightly inflated error rates.}
\label{tab:classification_vision_asymp}
\end{table*}

In Figure \ref{fig:classification_asymp} we show the realized budget save against the realized error when we use asymptotic intervals. Overall we see similar trends as in Figure \ref{fig:classification}, however the weaker requirement of asymptotic validity allows for generally larger saves.

\end{document}